\documentclass{article}
\usepackage[preprint]{log_2026}			

\usepackage{booktabs}						
\usepackage{multirow}						
\usepackage{amsfonts}						
\usepackage{graphicx}						
\usepackage{duckuments}						
\usepackage{adjustbox}
\usepackage{float}
\usepackage{array}
\usepackage{siunitx}
\usepackage{tabularx}
\usepackage{amsthm}

\newtheorem{theorem}{Theorem}[section]
\newtheorem{corollary}{Corollary}[section]
\newtheorem{lemma}{Lemma}[section]
\newtheorem{definition}{Definition}
\newcolumntype{M}{S[table-format=2.3]}
\newcolumntype{D}{>{\scriptsize}S[table-format=1.3]}

\usepackage[sort,round]{natbib}			

\title[HOPPER: Learnable Hop Extraction for Linearized Graph Sequence Models]{HOPPER: Learnable Hop Extraction for Linearized Graph Sequence Models}

\author[Herath et al.]{%
Isuru Herath \thanks{Equal contribution.}\\
Carnegie Mellon University \\
\email{iherath@andrew.cmu.edu}\And
Arin Gopakumar \footnotemark[1]\\
UC Berkeley\\
\email{aringopakumar@berkeley.edu}\And
Sharan Sahu \footnotemark[1] \\
Cornell University\\
\email{ss4329@cornell.edu}
}

\begin{document}
\raggedbottom

\maketitle

\begin{abstract}
Graph neural networks typically propagate information through repeated message-passing layers, coupling the distance over which information travels with the number of nonlinear transformations applied. This coupling can make deep architectures difficult to optimize and can lead to over-smoothing, over-squashing, and the loss of long-range information. Linearized Graph Sequence Models (LGSMs) \citep{mathys2026lgsm} address this issue by separating information propagation depth from processing depth and treating the successive propagation states of each node as a sequence. However, existing LGSMs construct these sequences using fixed graph operators, limiting their ability to adapt propagation to the input graph, node features, and downstream task. We introduce HOPPER, an end-to-end learnable extension of LGSM that learns how hop sequences should be extracted before they are processed by a modern state-space model. Our framework supports feature-conditioned, structure-aware, graph and hop-adaptive propagation mechanisms while preserving permutation equivariance, with standard adjacency-based and non-backtracking LGSM sequences arising as special cases of our extractor family. We show that HOPPER is state-of-the-art or competitive across the ECHO-Synth benchmark \citep{miglior2026can} and performs strongly on the City-Networks dataset \citep{liang2026towards}. On the LRIM physics-based long-range dependency benchmark \citep{mathys2026lrim}, we find that varying the maximum neighborhood size of message-backtracking cancellation, corresponding to the structural memory window, can substantially affect performance. Ablation studies further isolate the contributions of the learnable extraction mechanism and its structural and feature-adaptive components, demonstrating that adaptive hop-sequence construction is important beyond the downstream sequence model alone. Together, these results show that learnable sequence extraction provides a flexible and effective approach to long-range graph representation learning across synthetic, physics-based, and real-world graph benchmarks.
\end{abstract}

\section{Introduction}
\label{sec:intro}

Message-passing neural networks (MPNNs) are a standard architecture for graph learning \citep{gilmer2017neural,kipf2017semi}. By repeatedly aggregating local neighborhoods, they expand their receptive fields, but long-range propagation requires many layers, coupling two distinct computational dimensions: \emph{information depth}, how far information travels through the graph, and \emph{processing depth}, how many nonlinear transformations are applied. As depth increases, fixed-dimensional representations must compress increasingly large neighborhoods, leading to over-squashing, representation collapse, and optimization difficulties \citep{alon2021bottleneck,topping2022understanding,digiovanni2023power}, particularly for long-range graph tasks \citep{dwivedi2022longrange}. Meanwhile, modern sequence models, including Transformers \citep{vaswani2017attention} and selective state-space models such as Mamba and Mamba2 \citep{gu2024mamba,dao2024transformers}, efficiently model long-range dependencies. Applying them directly to graphs is challenging because graphs lack a canonical node ordering, requiring existing graph sequence models to introduce an additional serialization step through node prioritization, neighborhood tokenization, random walks, or learned permutations \citep{behrouz2024graph,wang2024graphmamba}.

Linearized Graph Sequence Models (LGSMs) provide an alternative approach to applying sequence models to graphs \citep{mathys2026lgsm}. Rather than treating the nodes of a graph as a sequence, LGSM treats the successive propagation messages from each node as a sequence over information depth. A state-space model then processes this sequence independently for each node, explicitly decoupling propagation depth from nonlinear processing depth. However, existing LGSMs construct these sequences using topology-dependent operators, such as normalized adjacency propagation or non-backtracking recurrences, with propagation rules that remain unchanged across graphs, features, tasks, and information depths. This can be limiting as the neighborhood range and propagation scale needed for effective learning often vary across problems \citep{xu2018representation,abu2019mixhop,chien2021adaptive,zhao2021adaptive}, potentially yielding sequences that discard useful structure before the downstream state-space model can exploit it.

We introduce \textsc{HOPPER}, an end-to-end learnable extension of LGSM that adaptively extracts hop sequences using feature-conditioned attention, graph and hop-dependent combinations of structural operators, and finite structural memory \citep{vaswani2017attention,ha2017hypernetworks,abu2019mixhop,chien2021adaptive}. The resulting sequence is processed by Mamba2-based SSM blocks \citep{dao2024transformers}, preserving permutation equivariance without graph serialization. The structural recurrence recovers non-backtracking propagation as a special case, while an augmented operator family also recovers normalized-adjacency and stable Chebyshev propagation. \textsc{HOPPER} achieves the best performance on eccentricity and single-source shortest-path prediction in \textsc{ECHO-Synth}, remains competitive on diameter and across City-Networks, and achieves its lowest \textsc{LRIM-16} test error among the evaluated memory settings with $M=8$.

\textbf{Related work.} Long-range graph learning has been approached by increasing message-passing depth, modifying graph connectivity to improve information flow, using graph transformers, and applying multi-hop graph filters \citep{klicpera2019diffusion,rong2020dropedge,zhao2020pairnorm,kreuzer2021rethinking,ying2021graphormer,rampasek2022gps,karhadkar2023fosr,shirzad2023exphormer}. Recent graph state-space models construct sequences through neighborhood tokenization, random walks, node prioritization, or learned permutations \citep{tonshoff2023graph,behrouz2024graph,wang2024graphmamba,kim2024graph}. HOPPER instead builds on LGSM and combines adaptive graph filtering, feature-conditioned propagation, hypernetworks, and permutation-invariant graph summarization \citep{ha2017hypernetworks,zaheer2017deepsets,velickovic2018graph,lee2019settransformer}. We survey the broader literature on long-range graph learning in Appendix~\ref{appendix:related_work}.

\section{Background}

\subsection{Notation}
\label{subsec:notation}

Let $\mathcal{G}=(\mathcal{V},\mathcal{E})$ be an undirected graph with $n=|\mathcal{V}|$ nodes and $m=|\mathcal{E}|$ edges, and let $\deg(v)$ denote the degree of $v\in\mathcal{V}$. We write $\mathbf{A}\in\mathbb{R}^{n\times n}$ for the adjacency matrix, where $\mathbf{A}_{uv}=1$ iff $(u,v)\in\mathcal{E}$, and $\mathbf{D}\in\mathbb{R}^{n\times n}$ for the diagonal degree matrix, with $\mathbf{D}_{vv}=\deg(v)$. The symmetrically and random-walk normalized adjacency matrices are $\mathbf{A}_{\mathrm{sym}}=\mathbf{D}^{-1/2}\mathbf{A}\mathbf{D}^{-1/2}$ and $\mathbf{A}_{\mathrm{rw}}=\mathbf{D}^{-1}\mathbf{A}$, respectively. Let $\mathbf{X}\in\mathbb{R}^{n\times F_{\mathrm{in}}}$ be the input node-feature matrix, with hidden representation $\mathbf{H}=\mathbf{X}\mathbf{W}_{\mathrm{in}}\in\mathbb{R}^{n\times d}$. For sequence length $L$, a hop sequence is $\mathbf{S}=(\mathbf{S}^{(0)},\ldots,\mathbf{S}^{(L-1)})$, where $\mathbf{S}^{(\ell)}\in\mathbb{R}^{n\times d}$ represents node states at information depth $\ell$. We use $[\mathbf{A}\|\mathbf{B}]$ for feature-wise concatenation and $\mathrm{vec}(\mathbf{A})$ for row-wise vectorization. Unless otherwise specified, attention $\mathrm{softmax}$ is applied across all keys for each query, $\mathrm{LayerNorm}$ is applied over the feature dimension, and $\mathrm{MLP}$ denotes a standard learnable multilayer perceptron.

\subsection{Graph Neural Networks}
Message-passing neural networks (MPNNs) \citep{gilmer2017neural} learn node
representations through repeated local neighborhood aggregation. Starting from
$\mathbf{H}^{(0)}=\mathbf{H}$, a message-passing layer updates the representation of node $v$ as
\begin{equation}
    \mathbf{h}_v^{(t+1)}
    =
    \psi^{(t)}
    \left(
        \mathbf{h}_v^{(t)},
        \phi^{(t)}
        \left(
            \left\{
                \mathbf{h}_u^{(t)} : u\in\mathcal{N}(v)
            \right\}
        \right)
    \right),
\end{equation}
where $\phi^{(t)}$ is a permutation-invariant neighborhood aggregator and
$\psi^{(t)}$ is a learnable nonlinear transformation. For common architectures
such as GCNs \citep{kipf2017semi}, this update takes the matrix form
\begin{equation}
    \mathbf{H}^{(t+1)}
    =
    \sigma \left(
        \mathbf{A}_{\mathrm{sym}}
        \mathbf{H}^{(t)}
        \mathbf{W}^{(t)}
    \right),
    \label{eq:mpnn}
\end{equation}
where $\mathbf{W}^{(t)}\in\mathbb{R}^{d\times d}$ is learnable and $\sigma$ is a point-wise nonlinearity. Each message-passing layer simultaneously propagates information one hop farther through the graph and applies another nonlinear transformation. Consequently, incorporating information from nodes at graph
distance $r$ requires at least $r$ message-passing layers, tightly coupling \emph{information depth} to \emph{processing depth}. Increasing the receptive
field therefore also increases the number of sequential transformations, which can make deep GNNs difficult to optimize and can exacerbate phenomena such as over-smoothing and the attenuation of long-range information.

\subsection{Linearized Graph Sequence Models}
\label{sec:lgsm}

Linearized Graph Sequence Models (LGSMs) \citep{mathys2026lgsm} decouple information propagation from nonlinear processing. Rather than mixing neighborhood aggregation and nonlinear transformations as in message-passing GNNs, LGSMs first construct, for each node, a sequence of representations at increasing information depths. Given $\mathbf{H}\in\mathbb{R}^{n\times d}$, a fixed graph propagator produces
\begin{equation}
    \mathbf{S}^{(\ell)}
    =
    \mathcal{P}_\ell(\mathbf{A},\mathbf{D})\mathbf{H},
    \;\;
    \ell=0,\ldots,L-1,
    \label{eq:lgsm-sequence}
\end{equation}
where $\mathcal{P}_\ell(\mathbf{A},\mathbf{D})\in\mathbb{R}^{n\times n}$
determines how information is propagated to depth $\ell$. For example, adjacency-power propagation uses $\mathbf{S}^{(\ell)} = \mathbf{A}_{\mathrm{sym}}^{\,\ell}\mathbf{H}$ so that $\mathbf{S}^{(\ell)}$ aggregates information from nodes within $\ell$ hops. The resulting states form $\mathbf{S}\in\mathbb{R}^{n\times L\times d}$, with the sequence associated with node $v$ given by $\mathbf{s}_v = \big(\mathbf{s}_v^{(0)},\ldots,\mathbf{s}_v^{(L-1)}\big)$, where $\mathbf{s}_v^{(\ell)}\in\mathbb{R}^d$ denotes the representation of node $v$ at hop $\ell$.

A shared sequence model $\Phi:\mathbb{R}^{L\times d}\rightarrow\mathbb{R}^{L\times d}$ is then
applied independently to each node sequence. Thus, the sequence length $L$ controls how far information propagates through the graph, while the depth and capacity of $\Phi$ control how that information is processed. However, existing LGSMs use fixed propagation rules across graphs and information depths, preventing the extracted sequences from adapting to graph structure, node features, or hop-specific propagation needs.

\section{Methodology}
\label{sec:method}

\subsection{Learnable Hop Extraction}
\label{sec:overview}

Linearized Graph Sequence Models (LGSMs) construct the hop sequence using a fixed propagation rule. For example, adjacency-based and non-backtracking LGSMs correspond to $\mathcal{P}_{\mathrm{adj}}=\{\mathbf{A}_{\mathrm{sym}}^{\,\ell}\}_{\ell=0}^{L-1}$ and $\mathcal{P}_{\mathrm{nb}}=\{\mathbf{B}^{(\ell)}\}_{\ell=0}^{L-1}$, respectively, where $\mathbf{B}^{(0)}=\mathbf{I}$, $\mathbf{B}^{(1)}=\mathbf{A}$, $\mathbf{B}^{(2)}=\mathbf{A}^2-\mathbf{D}$, and $\mathbf{B}^{(\ell)}=\mathbf{A}\mathbf{B}^{(\ell-1)}-(\mathbf{D}-\mathbf{I})\mathbf{B}^{(\ell-2)}$ for $\ell\geq 3$. Here, $\mathbf{I}\in\mathbb{R}^{n\times n}$ denotes the identity matrix. 

\textsc{HOPPER} retains the separation between graph propagation and sequence processing, but replaces the fixed LGSM extractor with a learnable, permutation-equivariant hop extractor. Given node features $\mathbf{X}\in\mathbb{R}^{n\times d_{\mathrm{in}}}$, we define
\begin{align*}
    &\mathbf{H} = \mathbf{X}\mathbf{W}_{\mathrm{in}}, \\
    &\mathbf{s}_v
    = \mathrm{SEQ}_{\theta}(\mathbf{H},\mathbf{A},\mathbf{D})_v
    = \big(\mathbf{s}_v^{(0)},\ldots,\mathbf{s}_v^{(L-1)}\big)
    \in \mathbb{R}^{L\times d},
    \;\; v\in\mathcal{V}, \\
    &\mathbf{z}_v
    = \Phi_{\omega}(\mathbf{s}_v)
    \in \mathbb{R}^{L\times d},
    \;\; v\in\mathcal{V}, \\
    &\hat{\mathbf{y}}
    = \rho(\mathbf{Z}),
    \label{eq:pipeline}
\end{align*}
where $\mathbf{S}=(\mathbf{s}_v)_{v\in\mathcal{V}}\in\mathbb{R}^{n\times L\times d}$ and $\mathbf{Z}=(\mathbf{z}_v)_{v\in\mathcal{V}}\in\mathbb{R}^{n\times L\times d}$. Here, $\mathrm{SEQ}_{\theta}$ constructs a length-$L$ hop sequence for each node, $\Phi_{\omega}:\mathbb{R}^{L\times d}\to\mathbb{R}^{L\times d}$ is a stack of $N_{\mathrm{blk}}$ sequence-processing blocks, and $\rho$ is a task-specific readout. Since $\Phi_{\omega}$ is applied independently to each node sequence, it processes information across hop depth without requiring an ordering of the graph nodes. The learnable extractor $\mathrm{SEQ}_{\theta}$ is built from four components: a linear structural recurrence, a permutation-invariant graph summary, graph-dependent and hop-dependent propagation coefficients, and a feature-attention correction applied to the extracted states.

\paragraph{Permutation-invariant structural graph summary.}
To condition propagation on the input graph without depending on node ordering or graph size, we construct a fixed-dimensional summary of its node features and local structure. We first define $\mathbf{H}_{\mathrm{nbr}}=\mathbf{A}_{\mathrm{rw}}\mathbf{H}$ and $\mathbf{H}_{\mathrm{str}}=[\mathbf{H}\mid\mathbf{H}_{\mathrm{nbr}}]\in\mathbb{R}^{n\times 2d}$. Let $\mathbf{Q}_{\mathrm{s}}\in\mathbb{R}^{K\times d}$ contain $K$ learnable seed queries, and define $\mathbf{K}_{\mathrm{s}}=\mathbf{H}_{\mathrm{str}}\mathbf{W}_{K}^{\mathrm{s}}$ and $\mathbf{V}_{\mathrm{s}}=\mathbf{H}_{\mathrm{str}}\mathbf{W}_{V}^{\mathrm{s}}$, where $\mathbf{W}_{K}^{\mathrm{s}},\mathbf{W}_{V}^{\mathrm{s}}\in\mathbb{R}^{2d\times d}$. The seed queries attend over the node representations according to
\begin{equation}
    \mathbf{H}_{\mathrm{lat}}
    =
    \mathrm{softmax}\left(
        \frac{\mathbf{Q}_{\mathrm{s}}\mathbf{K}_{\mathrm{s}}^{\top}}{\sqrt{d}}
    \right)
    \mathbf{V}_{\mathrm{s}}
    \in\mathbb{R}^{K\times d},
    \label{eq:pool-attention}
\end{equation}
where the softmax is applied over the node dimension. We then define $\mathbf{h}_{\mathrm{graph}}=\mathrm{vec}(\mathbf{H}_{\mathrm{lat}})\in\mathbb{R}^{Kd}$. Permuting the nodes permutes the rows of $\mathbf{K}_{\mathrm{s}}$ and $\mathbf{V}_{\mathrm{s}}$ identically, leaving $\mathbf{H}_{\mathrm{lat}}$ and hence $\mathbf{h}_{\mathrm{graph}}$ unchanged. The dimension $Kd$ is independent of the number of nodes.

\begin{figure}[H]
    \centering
    \includegraphics[width=\linewidth]{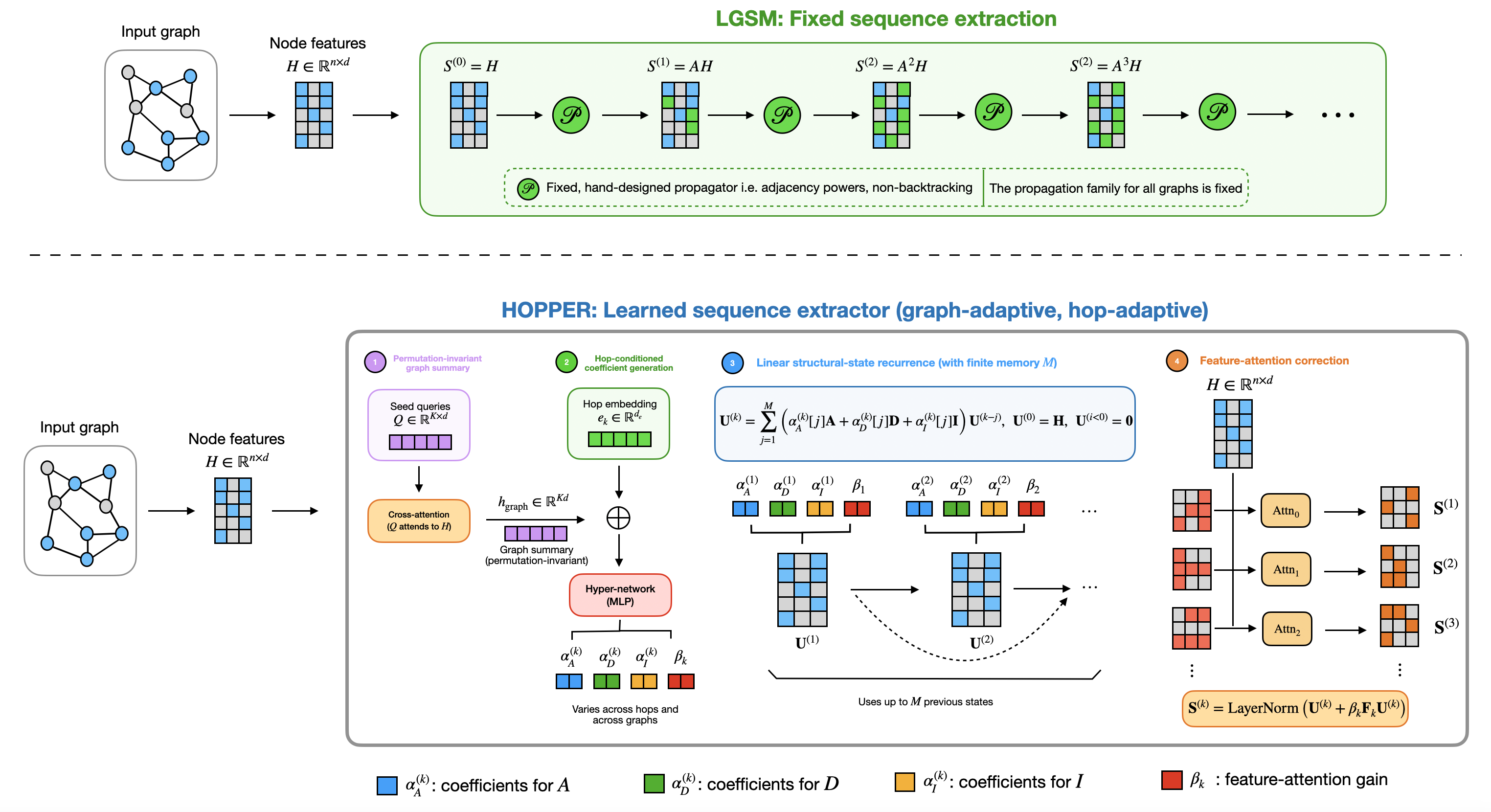}
    \caption{Standard LGSMs extract hop sequences using a fixed propagation rule shared across graphs and depths. HOPPER instead learns a graph and hop-adaptive sequence extractor through a permutation-invariant graph summary, hop-conditioned coefficients, a finite-memory linear structural recurrence, and an output-only feature-attention correction.}
    \label{fig:placeholder}
\end{figure}

\paragraph{Hop-conditioned coefficient generation.}
A learned step embedding $\mathbf{e}_k$ is concatenated with $\mathbf{h}_{\mathrm{graph}}$, and an MLP emits $3M+1$ scalars:
\begin{equation}
  [\boldsymbol{\alpha}_A^{(k)},\boldsymbol{\alpha}_D^{(k)},
  \boldsymbol{\alpha}_I^{(k)},\beta_k]
  =\mathrm{MLP}([\mathbf{h}_{\mathrm{graph}}\,\|\,\mathbf{e}_k])
  \in\mathbb{R}^{3M+1}.
  \label{eq:hyper-coeffs}
\end{equation}
Conditioning on the graph summary allows the propagation rule to vary across graphs, while the hop embedding allows it to vary with information depth.

\paragraph{Linear structural-state recurrence.}
The linear structural state is updated using a window of the previous $M$ states:
\begin{equation}
  \mathbf{U}^{(k)} = \sum_{j=1}^{M}
    \left(\alpha_A^{(k)}[j]\mathbf{A}+\alpha_D^{(k)}[j]\mathbf{D}
       +\alpha_I^{(k)}[j]\mathbf{I}\right)\mathbf{U}^{(k-j)},
  \;\;
  \mathbf{U}^{(0)}=\mathbf{H},\;\; \mathbf{U}^{(i<0)}=\mathbf{0}.
  \label{eq:shadow}
\end{equation}
Conditioned on the generated coefficients, each $\mathbf{U}^{(k)}$ is linear in $\mathbf{H}$ and can be expressed as a polynomial in $\mathbf{A}$, $\mathbf{D}$, and $\mathbf{I}$ applied to $\mathbf{H}$. We keep normalization, nonlinear activations, and attention outside the recurrence so that this polynomial structure is preserved and the fixed LGSM extractors can be recovered as special cases. Using the unnormalized operators $\mathbf{A}$ and $\mathbf{D}$ enables recovery of these fixed propagation rules, but repeated application can amplify state magnitudes as node degrees and hop depth increase. We discuss the resulting numerical stability considerations and an alternative normalized propagation operator in Appendix~\ref{app:normalized-path}.

\paragraph{Feature-attention correction.}
At each hop $k$, we compute an attention-based correction from the initial node representations $\mathbf{H}$. Let $\mathbf{Q}_k=\mathbf{H}\mathbf{W}_Q^{(k)}$ and $\mathbf{K}_k=\mathbf{H}\mathbf{W}_K^{(k)}$, and define
\begin{equation}
    \mathbf{F}_k
    =
    \mathrm{softmax}\left(
        \frac{\mathbf{Q}_k\mathbf{K}_k^\top}{\sqrt{d}}
    \right),
    \;\;
    \mathbf{S}^{(k)}
    =
    \mathrm{LayerNorm}\left(
        \mathbf{U}^{(k)}+\beta_k\mathbf{F}_k\mathbf{U}^{(k)}
    \right).
    \label{eq:feature-attention}
\end{equation}
The attention term modifies only the extracted state $\mathbf{S}^{(k)}$ and does not enter the structural recurrence in Eq.~\ref{eq:shadow}. This allows the extracted sequence to incorporate feature-dependent interactions while preserving the polynomial structure of $\mathbf{U}^{(k)}$.

\section{Theoretical Analysis}
\label{sec:theory}

In this section, we establish that \textsc{HOPPER} is permutation equivariant, admits graph-adaptive polynomial representations with finite-hop locality, recovers standard LGSM propagators as special cases, and can improve long-range sensitivity while preserving spectral information across information depth. All proofs are deferred to Appendix~\ref{sec:theoretical-properties}.

The first result shows that conditioning the sequence extractor on the input graph preserves permutation equivariance.

\begin{theorem}[Permutation equivariance of the sequence extractor]
\label{thm-main-body:permutation-equivariance}
Let $\mathbf{P}\in\mathbb{R}^{n\times n}$ be any permutation matrix, and define
$\mathbf{A}'=\mathbf{P}\mathbf{A}\mathbf{P}^{\top}$,
$\mathbf{D}'=\mathbf{P}\mathbf{D}\mathbf{P}^{\top}$, and
$\mathbf{H}'=\mathbf{P}\mathbf{H}$. Suppose the feature-attention maps are permutation equivariant, so that $\mathbf{Q}'_k=\mathbf{P}\mathbf{Q}_k$ and $\mathbf{K}'_k=\mathbf{P}\mathbf{K}_k$, and that $\mathrm{LayerNorm}$ is applied independently to each node with shared parameters. Then, for every $k\geq 0$,
$\mathbf{U}'^{(k)}=\mathbf{P}\mathbf{U}^{(k)}$ and
$\mathbf{S}'^{(k)}=\mathbf{P}\mathbf{S}^{(k)}$. Consequently,
\[
\mathrm{SEQ}_{\theta}
(\mathbf{P}\mathbf{H},
 \mathbf{P}\mathbf{A}\mathbf{P}^{\top},
 \mathbf{P}\mathbf{D}\mathbf{P}^{\top})
=
\mathbf{P}\,
\mathrm{SEQ}_{\theta}(\mathbf{H},\mathbf{A},\mathbf{D}),
\]
where $\mathbf{P}$ acts on the node dimension of the hop sequence.
\end{theorem}

Theorem~\ref{thm-main-body:permutation-equivariance} shows that relabeling the nodes leaves the graph summary and hypernetwork-generated propagation coefficients unchanged, while the structural recurrence and feature-attention correction transform equivariantly. Thus, the extracted sequence is relabeled consistently with the input graph. Combined with the shared nodewise sequence processor and a permutation-invariant readout, this also yields permutation-invariant graph-level predictions.

We next characterize the polynomial structure and finite-hop support of the structural recurrence.

\begin{theorem}[Adaptive polynomial structure and finite-hop structural support]
\label{thm-main-body:adaptive-polynomial-locality}
Fix a graph $\mathcal{G}=(\mathcal{V},\mathcal{E})$ and node representations $\mathbf{H}$, and let
$\{\boldsymbol{\alpha}_A^{(k)},\boldsymbol{\alpha}_D^{(k)},\boldsymbol{\alpha}_I^{(k)}\}_{k=1}^{L-1}$
denote the coefficients generated by the hypernetwork. Define
$\boldsymbol{\Gamma}^{(0)}=\mathbf{I}$,
$\boldsymbol{\Gamma}^{(i)}=\mathbf{0}$ for $i<0$, and, for $k\geq 1$,
\[
\boldsymbol{\Gamma}^{(k)}
=
\sum_{j=1}^{M}
\left(
\alpha_A^{(k)}[j]\mathbf{A}
+
\alpha_D^{(k)}[j]\mathbf{D}
+
\alpha_I^{(k)}[j]\mathbf{I}
\right)
\boldsymbol{\Gamma}^{(k-j)}.
\]
Then $\mathbf{U}^{(k)}=\boldsymbol{\Gamma}^{(k)}\mathbf{H}$ for every $k\geq 0$. Conditional on the generated coefficients, $\boldsymbol{\Gamma}^{(k)}$ is a noncommutative polynomial in $\mathbf{A}$ and $\mathbf{D}$ of degree at most $k$. Moreover, if $d_{\mathcal{G}}(u,v)>k$, then $[\boldsymbol{\Gamma}^{(k)}]_{uv}=0$. Consequently,
\[
\mathbf{U}^{(k)}_u
=
\sum_{v:\,d_{\mathcal{G}}(u,v)\leq k}
[\boldsymbol{\Gamma}^{(k)}]_{uv}\mathbf{H}_v.
\]
Thus, for fixed hypernetwork coefficients, the structural operator $\boldsymbol{\Gamma}^{(k)}$ is linear in $\mathbf{H}$ and has support within the $k$-hop neighborhood. Since the coefficients depend on the global graph summary $\mathbf{h}_{\mathrm{graph}}$, this support property does not imply that the end-to-end map $\mathbf{H}\mapsto\mathbf{U}^{(k)}$ is $k$-local.
\end{theorem}

Theorem~\ref{thm-main-body:adaptive-polynomial-locality} shows that, for fixed hypernetwork coefficients, each structural state is a finite-hop graph polynomial. The coefficients can still depend on global graph information, and the feature-attention correction can introduce additional global dependence in the full extractor.

We next relate the structural recurrence to fixed propagation rules used by existing LGSMs. For this analysis, we consider an augmented recurrence with an additional graph shift $\mathbf{G}$, while the original \textsc{HOPPER} recurrence is recovered by setting the corresponding coefficients to zero.

\begin{theorem}[Recovery of fixed LGSM extractors and stable Chebyshev propagation]
\label{thm-main-body:fixed-extractor-containment}
Suppose $M\geq 2$ and consider the recurrence
\[
\mathbf{U}^{(k)}
=
\sum_{j=1}^{M}
\left(
\alpha_G^{(k)}[j]\mathbf{G}
+
\alpha_A^{(k)}[j]\mathbf{A}
+
\alpha_D^{(k)}[j]\mathbf{D}
+
\alpha_I^{(k)}[j]\mathbf{I}
\right)\mathbf{U}^{(k-j)},
\;\;
\mathbf{U}^{(0)}=\mathbf{H},
\;\;
\mathbf{U}^{(i)}=\mathbf{0}\ \text{for }i<0.
\]
Then the following propagation schemes are recovered as special cases.

\begin{enumerate}
    \item[(i)] \textbf{Normalized adjacency propagation.}
    Setting $\alpha_G^{(k)}[1]=1$ for every $k\geq 1$ and all remaining coefficients to zero gives $\mathbf{U}^{(k)}=\mathbf{G}^{k}\mathbf{H}$. Taking $\mathbf{G}=\mathbf{A}_{\mathrm{sym}}$ recovers the adjacency based LGSM sequence $\mathcal{P}_{\mathrm{adj}}$.

    \item[(ii)] \textbf{Non-backtracking propagation.}
    Setting $\alpha_A^{(1)}[1]=1$, $\alpha_A^{(2)}[1]=1$, and $\alpha_D^{(2)}[2]=-1$, and for every $k\geq 3$ setting $\alpha_A^{(k)}[1]=1$, $\alpha_D^{(k)}[2]=-1$, and $\alpha_I^{(k)}[2]=1$, with all remaining coefficients equal to zero, gives $\mathbf{U}^{(k)}=\mathbf{B}^{(k)}\mathbf{H}$, where $\{\mathbf{B}^{(k)}\}_{k\geq 0}$ is the non-backtracking family defined in Section~\ref{sec:method}.

    \item[(iii)] \textbf{Chebyshev propagation.}
    Suppose additionally that $\mathbf{G}$ is symmetric and $\sigma(\mathbf{G})\subseteq[-1,1]$. Setting $\alpha_G^{(1)}[1]=1$, and for every $k\geq 2$ setting $\alpha_G^{(k)}[1]=2$ and $\alpha_I^{(k)}[2]=-1$, with all remaining coefficients equal to zero, gives $\mathbf{U}^{(k)}=T_k(\mathbf{G})\mathbf{H}$, 
    where $T_k$ is the degree-$k$ Chebyshev polynomial of the first kind. Moreover, $\|T_k(\mathbf{G})\|_2\leq 1$, and hence $\|\mathbf{U}^{(k)}\|_F\leq\|\mathbf{H}\|_F$ for every $k\geq 0$. In particular, one may take $\mathbf{G}=\mathbf{A}_{\mathrm{sym}}$.
\end{enumerate}
\end{theorem}

Theorem~\ref{thm-main-body:fixed-extractor-containment} shows that non-backtracking propagation is contained in the original structural recurrence, while the augmented operator family additionally recovers normalized-adjacency and Chebyshev propagation. Taking $\mathbf{G}\!=\!\mathbf{A}_{\mathrm{sym}}$ is convenient for the analysis below. Since $\mathbf{A}_{\mathrm{rw}}=\mathbf{D}^{-1/2}\mathbf{A}_{\mathrm{sym}}\mathbf{D}^{1/2}$, the two operators are similar, and under the change of coordinates $\widetilde{\mathbf{U}}^{(k)}=\mathbf{D}^{1/2}\mathbf{U}^{(k)}$, propagation by $\mathbf{A}_{\mathrm{rw}}$ is equivalent to propagation by $\mathbf{A}_{\mathrm{sym}}$. This lets us analyze long-range sensitivity and spectral behavior using a symmetric operator with spectrum in $[-1,1]$.

To compare long-range propagation fairly, we restrict attention to polynomial filters satisfying a common spectral stability constraint.

\begin{definition}[Uniformly stable polynomial propagators]
\label{def:stable-polynomial-propagators}
For $r\geq 1$, define
\[
\mathcal{P}_r
:=
\left\{
p\in\mathbb{R}[x]:
\deg(p)\leq r
\ \text{and}\
\sup_{x\in[-1,1]}|p(x)|\leq 1
\right\}.
\]
\end{definition}

If $\mathbf{G}$ is symmetric with $\sigma(\mathbf{G})\subseteq[-1,1]$, then every $p\in\mathcal{P}_r$ satisfies $\|p(\mathbf{G})\|_2\leq 1$. The following result characterizes the largest possible sensitivity between nodes at graph distance $r$ under this common stability constraint.

\begin{theorem}[Optimal stable long-range sensitivity]
\label{thm-main-body:optimal-long-range-sensitivity}
Let $\mathbf{G}\in\mathbb{R}^{n\times n}$ be symmetric with $\sigma(\mathbf{G})\subseteq[-1,1]$, and suppose $\mathbf{G}_{uv}=0$ whenever $u\neq v$ and $(u,v)\notin\mathcal{E}$. For nodes $u,v\in\mathcal{V}$ satisfying $d_{\mathcal{G}}(u,v)=r\geq 1$,
\begin{equation}
    \sup_{p\in\mathcal{P}_r}
    \left\|
        \frac{\partial[p(\mathbf{G})\mathbf{H}]_u}
        {\partial\mathbf{H}_v}
    \right\|_2
    =
    2^{r-1}|[\mathbf{G}^{r}]_{uv}|.
    \label{eq:optimal-long-range-sensitivity}
\end{equation}
The optimum is attained by the Chebyshev polynomial $p=T_r$. In comparison,
\[
    \left\|
        \frac{\partial[\mathbf{G}^{r}\mathbf{H}]_u}
        {\partial\mathbf{H}_v}
    \right\|_2
    =
    |[\mathbf{G}^{r}]_{uv}|.
\]
Thus, whenever $[\mathbf{G}^{r}]_{uv}\neq 0$, Chebyshev propagation improves first-arrival sensitivity by a factor of $2^{r-1}$ while remaining non-expansive.
\end{theorem}

At graph distance $r$, every polynomial term of degree less than $r$ has zero $(u,v)$ entry, so first-arrival sensitivity depends only on the leading coefficient. Among polynomials in $\mathcal{P}_r$, the Chebyshev polynomial has the largest possible leading coefficient, $2^{r-1}$, and therefore attains the optimum without increasing the spectral norm. For regular graphs, this improvement has a particularly simple form.

\begin{corollary}[Long-range sensitivity on regular graphs]
\label{cor:regular-graph-sensitivity}
Let $\mathcal{G}$ be a $q$-regular graph and let $\mathbf{G}=\mathbf{A}_{\mathrm{sym}}=q^{-1}\mathbf{A}$. If $d_{\mathcal{G}}(u,v)=r$ and $N_r(u,v)$ denotes the number of shortest paths of length $r$ from $v$ to $u$, then
\begin{equation}
    \left\|
        \frac{\partial[\mathbf{A}_{\mathrm{sym}}^r\mathbf{H}]_u}
        {\partial\mathbf{H}_v}
    \right\|_2
    =
    \frac{N_r(u,v)}{q^r},
    \;\;
    \left\|
        \frac{\partial[T_r(\mathbf{A}_{\mathrm{sym}})\mathbf{H}]_u}
        {\partial\mathbf{H}_v}
    \right\|_2
    =
    2^{r-1}\frac{N_r(u,v)}{q^r}.
    \label{eq:regular-graph-sensitivity}
\end{equation}
In particular, when $q=2$ and $N_r(u,v)=1$, power propagation has sensitivity $2^{-r}$ while Chebyshev propagation has sensitivity $1/2$ for every $r$.
\end{corollary}

Since Chebyshev propagation is realizable by the recurrence in Theorem~\ref{thm-main-body:fixed-extractor-containment}, these results show that the learned extractor can achieve the strongest possible first-arrival sensitivity among uniformly stable degree-$r$ polynomial propagators. The regular graph case further shows that the improvement over repeated adjacency propagation can be exponential in graph distance which can subsequently reduce long-range information attenuation associated with over-squashing.

We finally consider how much nonstationary spectral information remains available across a long hop sequence. Repeated normalized adjacency propagation suppresses eigendirections with eigenvalues of magnitude strictly smaller than one, while an LGSM provides every intermediate state to the downstream sequence model. This motivates measuring spectral information averaged over information depth.

\begin{theorem}[Preservation of spectral information across information depth]
\label{thm-main-body:spectral-information-preservation}
Let $\mathbf{A}_{\mathrm{sym}}=\mathbf{Q}\boldsymbol{\Lambda}\mathbf{Q}^{\top}$, and let $\boldsymbol{\Pi}$ denote the orthogonal projector onto the eigenspace associated with eigenvalue $1$. Assume every remaining eigenvalue satisfies $|\lambda_i|<1$, and define
$\mathbf{H}_{\perp}=(\mathbf{I}_n-\boldsymbol{\Pi})\mathbf{H}$ and
$\rho=\max_{\lambda_i\neq 1}|\lambda_i|<1$. For adjacency power propagation
$\mathbf{V}^{(\ell)}=\mathbf{A}_{\mathrm{sym}}^\ell\mathbf{H}$, for every $L\geq 1$,
\[
\frac{1}{L}
\sum_{\ell=0}^{L-1}
\|(\mathbf{I}_n-\boldsymbol{\Pi})\mathbf{V}^{(\ell)}\|_F^2
\leq
\frac{1-\rho^{2L}}{L(1-\rho^2)}
\|\mathbf{H}_{\perp}\|_F^2.
\]

For the Chebyshev sequence
$\mathbf{U}^{(\ell)}=T_\ell(\mathbf{A}_{\mathrm{sym}})\mathbf{H}$, for every $L\geq 1$,
\[
\left|
\frac{1}{L}
\sum_{\ell=0}^{L-1}
\|(\mathbf{I}_n-\boldsymbol{\Pi})\mathbf{U}^{(\ell)}\|_F^2
-
\frac{1}{2}\|\mathbf{H}_{\perp}\|_F^2
\right|
\leq
\frac{1}{2L\sqrt{1-\rho^2}}
\|\mathbf{H}_{\perp}\|_F^2.
\]
\end{theorem}

Theorem~\ref{thm-main-body:spectral-information-preservation} gives a finite-depth separation between adjacency-power and Chebyshev propagation. The average nonstationary spectral energy under adjacency powers decreases at rate $\mathcal{O}(1/L)$, while for the Chebyshev sequence it remains within $\mathcal{O}(1/L)$ of $\frac{1}{2}\|\mathbf{H}_{\perp}\|_F^2$. Thus, Chebyshev propagation preserves nonstationary spectral information across the hop sequence even at finite information depths.

This result also connects directly to over-smoothing through the Dirichlet energy \citep{arroyo2026vanishing,rusch2023survey}. For $\mathbf{H}\in\mathbb{R}^{n\times d}$, define $\mathcal{E}(\mathbf{H})=\operatorname{Tr}(\mathbf{H}^{\top}\boldsymbol{\Delta}\mathbf{H})$, where $\boldsymbol{\Delta}$ is the normalized graph Laplacian \citep{chung1997spectral}. For an unweighted graph, $\mathcal{E}(\mathbf{H})$ measures the variation between normalized representations of neighboring nodes and approaches zero as these representations become increasingly similar. Theorem~\ref{thm-main-body:spectral-information-preservation} implies that $\mathcal{E}(\mathbf{A}_{\mathrm{sym}}^k\mathbf{H})\to 0$, whereas the average Dirichlet energy of the Chebyshev sequence converges to $\frac{1}{2}\mathcal{E}(\mathbf{H})$. Thus, when $\mathcal{E}(\mathbf{H})>0$, Chebyshev propagation preserves nonstationary spectral variation across information depth that vanishes under repeated adjacency propagation (see Appendix~\ref{sec:theoretical-properties} for more details).

\section{Experimental Evaluation}
\label{sec:experiments}

We evaluate \textsc{HOPPER} on \textsc{ECHO-Synth}~\citep{miglior2026echo} and \textsc{City-Networks}~\citep{liang2026towards}. \textsc{ECHO-Synth} contains three long-range graph prediction tasks: graph diameter (\textsc{Diam}), node eccentricity (\textsc{Ecc}), and single-source shortest path (\textsc{Sssp}). \textsc{City-Networks} contains road networks from Paris, Shanghai, Los Angeles, and London and evaluates node classification using eccentricity-based accessibility labels.

We compare \textsc{HOPPER} with message-passing GNNs, multi-hop and dynamics-inspired methods, graph transformers, and LGSM~\citep{mathys2026lgsm}, our closest sequence-based baseline on \textsc{ECHO-Synth}. For \textsc{City-Networks}, we replace dense feature attention with WildCat~\citep{schroder2026wildcat} to avoid quadratic computation and memory on graphs with hundreds of thousands of nodes. We follow the evaluation protocol of each benchmark. Additional dataset statistics, model and optimization settings, and details of the architectural and \textsc{LRIM-16} ablations are provided in Appendix~\ref{appendix:additional-methodological}.

\begin{table}[H]
\captionsetup{skip=6pt}
\caption{
Performance on the \textsc{ECHO-Synth} long-range graph prediction tasks. We report test MSE and MAE for diameter (\textsc{Diam}), eccentricity (\textsc{Ecc}), and single-source shortest-path (\textsc{Sssp}) prediction. Results are reported as mean $\pm$ standard deviation over three seeds.
}
\label{tab:echo-results}
\centering

\begingroup
\small
\setlength{\tabcolsep}{3.25pt}
\renewcommand{\arraystretch}{1.08}

\begin{adjustbox}{max width=\linewidth}
\begin{tabular}{
    @{}l
    M@{\,{\scriptsize$\pm$}\,}D
    M@{\,{\scriptsize$\pm$}\,}D
    M@{\,{\scriptsize$\pm$}\,}D
    M@{\,{\scriptsize$\pm$}\,}D
    M@{\,{\scriptsize$\pm$}\,}D
    M@{\,{\scriptsize$\pm$}\,}D
    @{}
}
\toprule
\multicolumn{1}{c}{\textsc{Model}}
& \multicolumn{4}{c}{\textsc{Diam}}
& \multicolumn{4}{c}{\textsc{Ecc}}
& \multicolumn{4}{c}{\textsc{Sssp}} \\
\cmidrule(lr){2-5}
\cmidrule(lr){6-9}
\cmidrule(l){10-13}

& \multicolumn{2}{c}{MSE $\downarrow$}
& \multicolumn{2}{c}{MAE $\downarrow$}
& \multicolumn{2}{c}{MSE $\downarrow$}
& \multicolumn{2}{c}{MAE $\downarrow$}
& \multicolumn{2}{c}{MSE $\downarrow$}
& \multicolumn{2}{c}{MAE $\downarrow$} \\
\midrule

A-DGN
& 4.818  & 0.108
& 1.151  & 0.038
& 35.967 & 0.492
& 4.981  & 0.037
& 4.425  & 0.879
& 1.176  & 0.140 \\

DRew
& 3.756  & 0.170
& 1.243  & 0.047
& 32.247 & 0.148
& 4.651  & 0.020
& 6.589  & 0.015
& 1.279  & 0.011 \\

GCN
& 22.872 & 2.766
& 3.832  & 0.262
& 39.706 & 0.460
& 5.233  & 0.034
& 9.743  & 0.757
& 2.102  & 0.094 \\

GCNII
& 9.696  & 0.568
& 2.005  & 0.093
& 39.911 & 0.518
& 5.241  & 0.030
& 10.369 & 3.575
& 2.128  & 0.429 \\

GIN
& 7.238  & 1.153
& 1.630  & 0.161
& 34.454 & 1.201
& 4.869  & 0.092
& 11.868 & 2.689
& 2.234  & 0.271 \\

GPS
& 10.454 & 0.610
& 2.160  & 0.098
& 33.346 & 0.226
& 4.758  & 0.021
& 1.255  & 0.113
& 0.472  & 0.050 \\

GRIT
& 3.877  & 0.295
& 1.014  & 0.046
& 38.667 & 1.903
& 5.091  & 0.158
& 0.147  & 0.083
& 0.121  & 0.013 \\

GraphCON
& 16.427 & 1.419
& 2.969  & 0.189
& 43.505 & 0.017
& 5.474  & 0.001
& 52.104 & 0.016
& 5.734  & 0.011 \\

PH-DGN
& 6.699  & 2.728
& 1.627  & 0.398
& 37.510 & 1.416
& 5.068  & 0.126
& 4.656  & 3.013
& 1.323  & 0.485 \\

SWAN
& 4.950  & 0.265
& 1.121  & 0.070
& 34.208 & 0.578
& 4.840  & 0.045
& 2.905  & 1.556
& 0.896  & 0.232 \\

LGSM
& {\bfseries 3.089}  & {\bfseries 0.389}
& {\bfseries 0.859}  & {\bfseries 0.044}
& 13.549 & 0.539
& 2.430  & 0.070
& 0.040  & 0.008
& 0.076  & 0.009 \\

\midrule

\textbf{HOPPER (ours)}
& 3.200 & 0.528
& 0.973 & 0.040
& {\bfseries 6.571} & {\bfseries 2.656}
& {\bfseries 1.820} & {\bfseries 0.386}
& {\bfseries 0.021} & {\bfseries 0.008}
& {\bfseries 0.035} & {\bfseries 0.002} \\

\bottomrule
\end{tabular}
\end{adjustbox}

\endgroup
\end{table}

\begin{table}[H]
\captionsetup{skip=6pt}
\caption{
Performance on node classification tasks from \textsc{City-Networks}. We report test accuracy (\%) for Paris, Shanghai, Los Angeles, and London. Results are reported as mean $\pm$ standard deviation over five seeds.
}
\label{tab:city-results}
\centering

\begingroup
\small
\setlength{\tabcolsep}{5pt}
\renewcommand{\arraystretch}{1.08}

\begin{adjustbox}{max width=\linewidth}
\begin{tabular}{@{}lcccc@{}}
\toprule
\multicolumn{1}{c}{\textsc{Model}}
& \textsc{Paris}
& \textsc{Shanghai}
& \textsc{Los Angeles}
& \textsc{London} \\
\cmidrule(lr){2-2}
\cmidrule(lr){3-3}
\cmidrule(lr){4-4}
\cmidrule(l){5-5}

& Accuracy $\uparrow$
& Accuracy $\uparrow$
& Accuracy $\uparrow$
& Accuracy $\uparrow$ \\
\midrule

MLP
& $25.5 \pm 0.4$
& $28.4 \pm 0.6$
& $24.1 \pm 0.5$
& $27.9 \pm 0.1$ \\

ChebNet
& $54.1 \pm 0.2$
& $66.5 \pm 0.1$
& $61.4 \pm 0.4$
& $54.7 \pm 0.2$ \\

GCN
& $53.2 \pm 0.3$
& $62.1 \pm 0.2$
& $58.3 \pm 0.3$
& $50.1 \pm 0.7$ \\

GraphSAGE
& $54.6 \pm 0.2$
& $68.3 \pm 0.5$
& $61.4 \pm 0.3$
& $55.4 \pm 0.2$ \\

GAT
& $51.1 \pm 0.3$
& $68.0 \pm 0.5$
& $59.5 \pm 0.3$
& $52.0 \pm 0.3$ \\

GCNII
& $51.3 \pm 0.2$
& $61.5 \pm 0.4$
& $56.0 \pm 0.3$
& $48.2 \pm 0.3$ \\

DropEdge
& $48.2 \pm 0.2$
& $60.8 \pm 0.4$
& $55.5 \pm 0.3$
& $45.0 \pm 0.3$ \\

GraphGPS
& $52.1 \pm 0.6$
& $63.0 \pm 0.5$
& $59.8 \pm 0.5$
& OOM \\

Exphormer
& $\mathbf{55.1 \pm 0.8}$
& $\mathbf{70.2 \pm 0.4}$
& $\mathbf{63.8 \pm 0.6}$
& $49.5 \pm 0.4$ \\

SGFormer
& $52.0 \pm 0.8$
& $64.1 \pm 0.3$
& $60.1 \pm 0.7$
& $48.3 \pm 0.3$ \\

\midrule

\textbf{HOPPER (ours)}
& $54.8 \pm 0.6$
& $69.4 \pm 0.5$
& $61.1 \pm 0.3$
& $\mathbf{55.5 \pm 0.2}$ \\

\bottomrule
\end{tabular}
\end{adjustbox}

\endgroup
\end{table}

Tables~\ref{tab:echo-results} and~\ref{tab:city-results} summarize the main results. \textsc{HOPPER} achieves the best performance on \textsc{Ecc} and \textsc{Sssp} and remains competitive with LGSM on \textsc{Diam}. On \textsc{City-Networks}, \textsc{HOPPER} remains competitive across all four cities and achieves the best result on London. These results show that \textsc{HOPPER} remains competitive on large real-world graphs in addition to its gains on the \textsc{ECHO-Synth} tasks.

We next study the main architectural components of \textsc{HOPPER}. Figure~\ref{fig:ecc-ablations} ablates neighborhood information, the adaptive hypernetwork, and the feature-attention correction on \textsc{ECHO-Synth} eccentricity, a node-level prediction task. Removing each component degrades performance, with the largest decrease occurring when the hypernetwork is replaced by fixed Chebyshev propagation. Figure~\ref{fig:window-size} varies the structural memory window on \textsc{LRIM-16}, where performance improves over the evaluated settings and is best at $M=8$. Additional ablation details are provided in Appendix~\ref{appendix:additional-methodological}, including ablation performance on the graph-level diameter prediction task (Figure \ref{fig:diam-ablations}).

\begin{figure}[H]
    \centering
    \includegraphics[width=\linewidth]{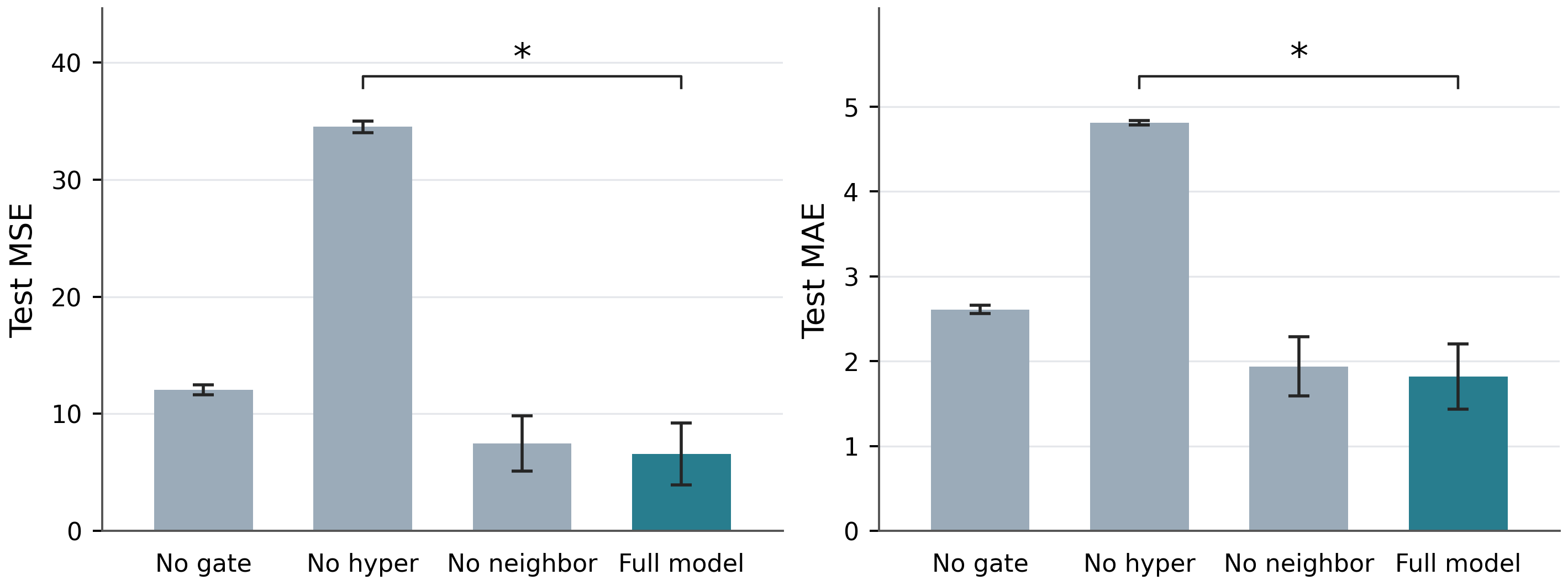}
    \caption{
    \textsc{ECHO-Synth} eccentricity ablations. Bars show mean test error over three seeds, with error bars indicating $\pm 1$ sample standard deviation. An asterisk denotes $p\leq 0.05$ from a two-sided paired $t$-test against the full \textsc{HOPPER} model.
    }
    \label{fig:ecc-ablations}
\end{figure}

\begin{figure}[H]
    \centering
    \includegraphics[width=0.6\linewidth]{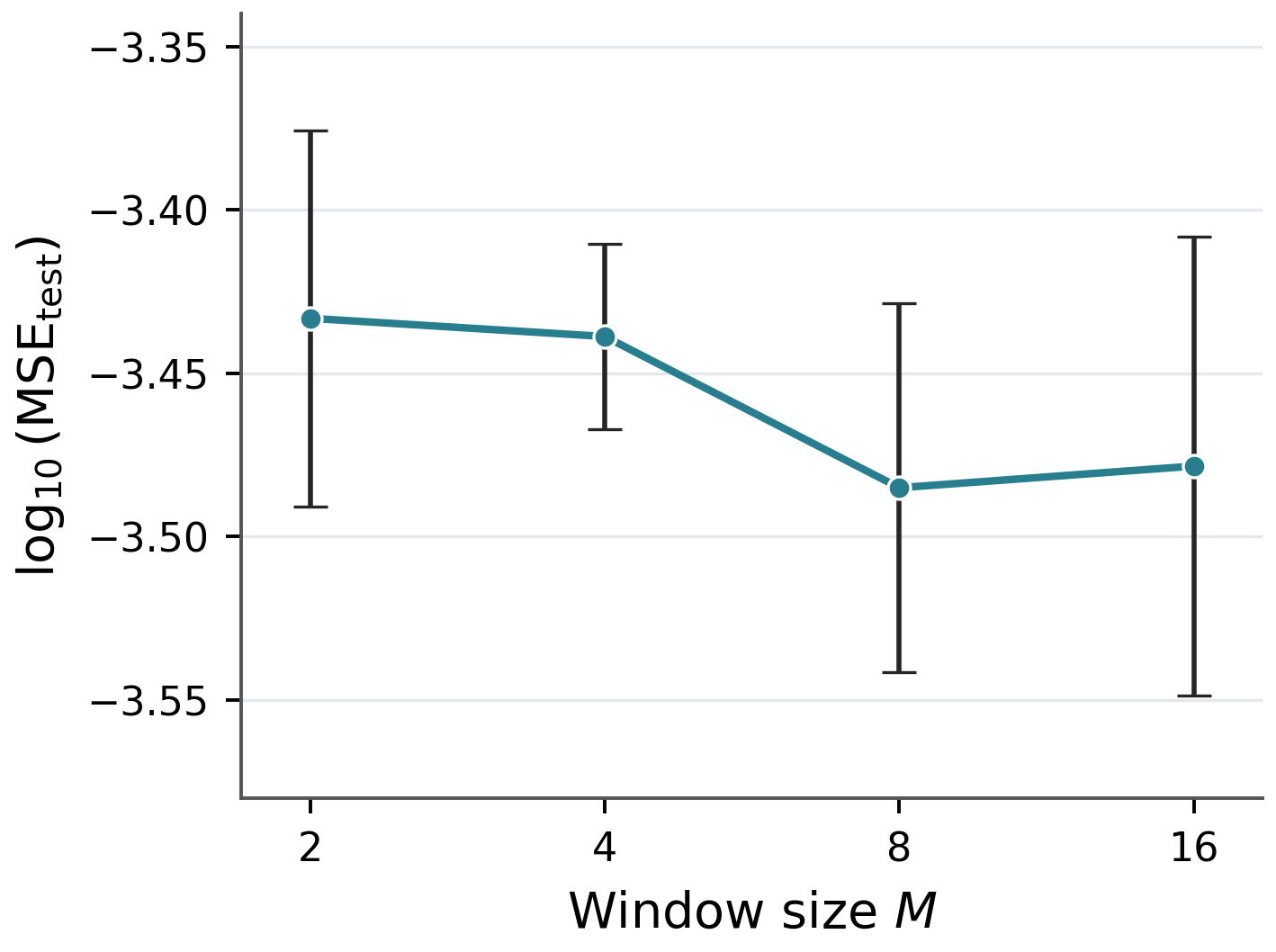}
    \caption{
    \textsc{LRIM-16} test performance across structural memory windows $M\in\{2,4,8, 16\}$. Points show mean $\log_{10}(\mathrm{MSE})$ over three seeds, with error bars indicating $\pm 1$ sample standard deviation. Lower is better.
    }
    \label{fig:window-size}
\end{figure}

Overall, our experimental results show that \textsc{HOPPER} benefits from adaptive hop extraction and structural memory across both synthetic long-range tasks and large real-world graphs.

\section{Conclusion}
\label{sec:conclusion}
We introduced \textsc{HOPPER}, a learnable extension of LGSM that adapts hop sequence extraction to the input graph and information depth before state-space processing. \textsc{HOPPER} preserves permutation equivariance and finite-hop locality, recovers standard LGSM propagators as special cases, and supports stable polynomial propagation with improved long-range sensitivity and spectral preservation. On \textsc{ECHO-Synth}, \textsc{HOPPER} achieves the best performance on eccentricity and single-source shortest-path prediction while remaining competitive on diameter. On \textsc{City-Networks}, it remains competitive on large real-world road networks using WildCat attention and achieves the best performance on London. Together, these results show that learning how graph information is propagated across hop depth provides a flexible alternative to fixed LGSM extractors. Future work includes extending \textsc{HOPPER} to additional real-world domains, including molecular graphs. Code is available at \url{https://anonymous.4open.science/r/HOPPER-0593/}.

\subsection*{AI use statement}
In this work, we used generative AI tools to assist with verifying proofs and mathematical statements and running and debugging experiments. We did not use generative AI tools to generate synthetic datasets, develop the core theoretical framework, formulate the central hypotheses, or design the primary research methodology. Other required disclosure tasks were not applicable to this work. Additionally, we used generative AI tools to assist with editing the manuscript for clarity, grammar, and readability, as well as with writing and debugging research code.

We reviewed all AI-assisted work. In particular, all AI-assisted experimental code and results were reviewed and tested for correctness. All final scientific claims, theoretical results, experimental conclusions, and manuscript text were determined and approved by the authors. We take responsibility for the final content of this work, including text, claims, code, and artifacts produced with the aid of generative AI.

\subsection*{Ethics statement}
This paper presents work whose goal is to advance the field of graph representation learning. There are many potential societal consequences of our work, none of which we feel must be specifically highlighted here.

\subsection*{Reproducibility statement}

We have taken several steps to facilitate reproducibility of our theoretical and empirical results. The model architecture and sequence-extraction procedure are specified in the main text, and complete proofs of all theoretical results are provided in Appendix~\ref{sec:theoretical-properties}. Appendix~\ref{appendix:additional-methodological} provides detailed experimental specifications, including dataset descriptions and splits, model architectures, optimization settings, hyperparameters, training procedures, evaluation metrics, and hardware configurations for \textsc{ECHO-Synth}, \textsc{LRIM}, and City-Networks. We also report results over multiple random seeds using consistent training budgets and model-selection procedures. Lastly, we have made our code available at \url{https://anonymous.4open.science/r/HOPPER-0593/}. Together, these details are intended to provide sufficient information to reproduce the theoretical and empirical results reported in this work.

\bibliography{iclr2026_conference}
\bibliographystyle{iclr2026_conference}

\newpage

\appendix

\section{Related Work}
\label{appendix:related_work}
\noindent\textbf{Learning long-range dependencies on graphs.}
Although MPNNs effectively model local structure, they often struggle with long-range dependencies because of over-smoothing, over-squashing, and vanishing gradients \citep{li2018deeper,oono2020graph,alon2021bottleneck,topping2022understanding,digiovanni2023power}. These limitations affect standard architectures such as GCN, GraphSAGE, and GIN on tasks requiring global context \citep{kipf2017semi,hamilton2017inductive,xu2019powerful,dwivedi2022longrange}. Proposed remedies include deeper residual architectures and normalization methods \citep{rong2020dropedge,zhao2020pairnorm,chen2020simple}, graph rewiring and diffusion \citep{klicpera2019diffusion,karhadkar2023fosr,gutteridge2023drew}, and Graph Transformers such as SAN, Graphormer, GPS, and Exphormer \citep{kreuzer2021rethinking,ying2021graphormer,rampasek2022gps,shirzad2023exphormer}. These methods improve communication between distant nodes through modified connectivity, global attention, or more stable deep computation.

\noindent\textbf{Multi-hop and adaptive graph propagation.}
A related line of work represents graph information across several neighborhood scales. APPNP separates feature transformation from personalized PageRank propagation \citep{gasteiger2019predict}, while Jumping Knowledge Networks combine representations from different message-passing depths \citep{xu2018representation}. MixHop and SIGN explicitly construct features from multiple powers or diffusions of the graph operator \citep{abu2019mixhop,frasca2020sign}, and DAGNN adaptively combines representations obtained at different propagation depths \citep{liu2020deeper}. GPR-GNN and BernNet learn flexible graph filters through generalized PageRank and Bernstein polynomial parameterizations \citep{chien2021adaptive,he2021bernnet}, while adaptive diffusion methods learn task-dependent propagation behavior \citep{zhao2021adaptive}. Feature-conditioned methods such as GAT and GNN-FiLM further adapt local message passing to node representations \citep{velickovic2018graph,brockschmidt2020gnnfilm}. Collectively, these works show that the useful propagation operator and neighborhood range are strongly dependent on the data and task.

\noindent\textbf{State-space and sequence models for graphs.}
Structured state-space models such as S4, LRU, Mamba, and Mamba2 provide efficient alternatives to attention for long-range sequence modeling \citep{gu2021efficiently,orvieto2023resurrecting,gu2024mamba,dao2024transformers}. Recent graph architectures adapt these models by constructing sequences through neighborhood tokenization, random walks, node prioritization, or learned permutations \citep{tonshoff2023graph,behrouz2024graph,wang2024graphmamba,kim2024graph}. Hybrid Graph Sequence Models combine graph-specific local encoders with global sequence processors \citep{behrouz2025hybrid}. Other approaches incorporate state-space computation directly into graph architectures: Graph State Space Convolution constructs permutation-equivariant graph operators using state-space kernels and invariant aggregation \citep{huang2024gssc}, while Message-Passing State-Space Models embed recurrent state-space dynamics within message passing and analyze their long-range information flow \citep{ceni2025mpssm}. These works explore different choices for defining a sequence axis while retaining graph structure and permutation symmetries.

\noindent\textbf{Linearized Graph Sequence Models.}
Linearized Graph Sequence Models interpret the successive propagation states of each node as a sequence over information depth \citep{mathys2026lgsm}. This formulation separates information depth from nonlinear processing depth and allows modern sequence models to process node-wise propagation trajectories without imposing an ordering over graph nodes. LGSM considers normalized-adjacency and non-backtracking sequence extractors and studies how their design affects sensitivity, stability, and information flow. Adjacency-based extraction is closely connected to polynomial graph filters and diffusion-based propagation \citep{defferrard2016convolutional,klicpera2019diffusion,chien2021adaptive}, while non-backtracking extraction draws on classical operators from spectral graph theory and community detection \citep{hashimoto1989zeta,krzakala2013spectral,bordenave2015nonbacktracking}. Our architecture builds on the LGSM framework and draws from adaptive graph filtering, feature-conditioned propagation, hypernetworks, and permutation-invariant graph summarization \citep{ha2017hypernetworks,zaheer2017deepsets,velickovic2018graph,lee2019settransformer}.

\newpage
\section{Experimental Details}
\label{appendix:additional-methodological}

\noindent\textbf{Operator Paths for the Structural Recurrence.}
\label{app:normalized-path}
The structural recurrence in Eq.~\ref{eq:shadow} uses the unnormalized adjacency and degree matrices $\mathbf{A}$ and $\mathbf{D}$. We refer to this as the \emph{exact} path because it preserves the polynomial structure analyzed in Theorem~\ref{thm-main-body:adaptive-polynomial-locality} and exactly recovers the non-backtracking propagation rule described in Section~\ref{sec:method}. At large information depths, however, repeated application of $\mathbf{A}$ and $\mathbf{D}$ can cause the structural states to grow rapidly with node degree and hop index, leading to poor numerical conditioning. For longer sequences, we therefore replace $\mathbf{A}$ with the random-walk matrix $\mathbf{A}_{\mathrm{rw}}=\mathbf{D}^{-1}\mathbf{A}$ and $\mathbf{D}$ with $\mathbf{I}-\mathbf{D}^{-1}$. This gives the normalized recurrence
\begin{equation}
\mathbf{U}^{(k)}
=
\sum_{j=1}^{M}
\left(
\alpha_A^{(k)}[j]\mathbf{A}_{\mathrm{rw}}
+
\alpha_D^{(k)}[j](\mathbf{I}-\mathbf{D}^{-1})
+
\alpha_I^{(k)}[j]\mathbf{I}
\right)
\mathbf{U}^{(k-j)},
\label{eq:shadow-norm}
\end{equation}
with the same initial conditions as Eq.~\ref{eq:shadow}. These normalized operators reduce the degree-dependent scaling of the structural recurrence and improve numerical stability at larger information depths. The experiments in Section~\ref{sec:experiments} use Eq.~\ref{eq:shadow-norm}, with $L=40$ for \textsc{ECHO-Synth},  $L=32$ for \textsc{LRIM}, and $L=17$ for City-Networks.

All models use Mamba2~\citep{dao2024transformers} as the sequence-processing backbone. Unless otherwise specified, we train with AdamW and gradient clipping with maximum norm $1.0$. Model selection is based on validation loss, and all reported results are evaluated over multiple random seeds using the same training budget. All experiments are run on an array of $16$ 32GB NVIDIA V100 GPUs. Dataset-specific architectures, optimization settings, hardware configurations, and evaluation metrics are provided below.

\noindent\textbf{Dataset Summaries.}
\label{app:datasets}
We evaluate \textsc{HOPPER} on three long-range graph benchmarks. \textsc{ECHO-Synth}~\citep{miglior2026echo} contains three synthetic prediction tasks: graph diameter (\textsc{ECHO-DIAM}), node eccentricity (\textsc{ECHO-ECC}), and single-source shortest path (\textsc{ECHO-SSSP}). \textsc{LRIM-16}~\citep{mathys2026lrim} consists of $16\times16$ periodic Ising grids with varying spin configurations and evaluates node-level energy prediction under long-range interactions. \textsc{City-Networks}~\citep{liang2026towards} contains road networks from Paris, Shanghai, Los Angeles, and London, with road junctions represented as nodes and road segments as edges. The task is node classification using eccentricity-based accessibility labels. Tables~\ref{tab:dataset-summary} and~\ref{tab:dataset-stats} summarize the datasets and their graph statistics, with additional \textsc{City-Networks} statistics reported in Table~\ref{tab:city-stats}.

\begin{table}[H]
\captionsetup{skip=6pt}
\caption{
Datasets used in our experiments. \textsc{ECHO-Synth} contains synthetic graph prediction tasks, \textsc{LRIM-16} contains periodic Ising grids with long-range interactions, and \textsc{City-Networks} contains large-scale urban road networks.
}
\label{tab:dataset-summary}
\centering

\begingroup
\small
\setlength{\tabcolsep}{4pt}
\renewcommand{\arraystretch}{1.08}

\begin{tabularx}{\linewidth}{
    @{}
    >{\raggedright\arraybackslash}p{0.17\linewidth}
    >{\raggedright\arraybackslash}X
    >{\raggedright\arraybackslash}p{0.30\linewidth}
    @{}
}
\toprule
\textsc{Dataset}
& \textsc{Structure / Node Input}
& \textsc{Prediction Target} \\
\midrule

\textsc{ECHO-Synth}~\citep{miglior2026echo}
& Synthetic graph families with random scalar node features and a source indicator for \textsc{SSSP}
& Graph diameter, node eccentricity, and \textsc{SSSP} distance \\

\textsc{LRIM-16} (hard)~\citep{mathys2026lrim}
& $16\times16$ periodic Ising grid with binary spin states
& Node-wise energy change $\Delta E_i$ \\

\textsc{City-Networks}~\citep{liang2026towards}
& City road networks with road junctions as nodes and road segments as edges
& Node accessibility class derived from eccentricity \\

\bottomrule
\end{tabularx}

\endgroup
\end{table}

\begin{table}[H]
\captionsetup{skip=6pt}
\caption{
Basic graph statistics for \textsc{ECHO-Synth}~\citep{miglior2026echo},
\textsc{LRIM-16}~\citep{mathys2026lrim}, and
\textsc{City-Networks}~\citep{liang2026towards}.
\textsc{ECHO-Synth} values are reported as mean $\pm$ standard deviation.
}
\label{tab:dataset-stats}
\centering

\begingroup
\small
\setlength{\tabcolsep}{5pt}
\renewcommand{\arraystretch}{1.08}

\begin{tabular*}{\linewidth}{@{\extracolsep{\fill}}lrrrrr@{}}
\toprule
\textsc{Dataset}
& \# Graphs
& Nodes
& Edges
& Avg. Degree
& Diameter \\
\midrule

\textsc{ECHO-Synth}
& 10{,}080
& $83.69 \pm 66.24$
& $211.63 \pm 209.39$
& $2.53 \pm 1.19$
& $28.50 \pm 6.92$ \\

\quad \texttt{line}
& 1{,}680
& $75.60 \pm 27.32$
& $90.10 \pm 33.89$
& $2.37 \pm 0.10$
& $28.50 \pm 6.92$ \\

\quad \texttt{ladder}
& 1{,}680
& $56.52 \pm 13.82$
& $82.54 \pm 20.72$
& $2.92 \pm 0.02$
& $28.50 \pm 6.92$ \\

\quad \texttt{grid}
& 1{,}680
& $193.10 \pm 93.10$
& $288.32 \pm 145.29$
& $2.95 \pm 0.12$
& $28.50 \pm 6.92$ \\

\quad \texttt{tree}
& 1{,}680
& $60.42 \pm 17.17$
& $59.42 \pm 17.17$
& $1.96 \pm 0.01$
& $28.50 \pm 6.92$ \\

\quad \texttt{caterpillar}
& 1{,}680
& $34.71 \pm 7.96$
& $33.71 \pm 7.96$
& $1.94 \pm 0.02$
& $28.50 \pm 6.92$ \\

\quad \texttt{lobster}
& 1{,}680
& $81.79 \pm 25.46$
& $80.79 \pm 25.46$
& $1.97 \pm 0.01$
& $28.50 \pm 6.92$ \\

\midrule

\textsc{LRIM-16} (hard)
& 10{,}000
& 256
& 512
& 4.00
& 16 \\

\midrule

\textsc{City-Networks}
& 4
& 
& 
& 
& \\

\quad Paris
& 1
& 114k
& 183k
& 3.2
& 121 \\

\quad Shanghai
& 1
& 184k
& 263k
& 2.9
& 123 \\

\quad Los Angeles
& 1
& 241k
& 343k
& 2.8
& 171 \\

\quad London
& 1
& 569k
& 759k
& 2.7
& 404 \\

\bottomrule
\end{tabular*}

\endgroup
\end{table}

\begin{table}[H]
\captionsetup{skip=6pt}
\caption{
Additional structural statistics for the four graphs in
\textsc{City-Networks}~\citep{liang2026towards}.
}
\label{tab:city-stats}
\centering

\begingroup
\small
\setlength{\tabcolsep}{8pt}
\renewcommand{\arraystretch}{1.08}

\begin{tabular*}{0.85\linewidth}{@{\extracolsep{\fill}}lrrrrr@{}}
\toprule
\textsc{City}
& Degree SD
& Max Degree
& Clustering
& Transitivity
& Homophily \\
\midrule

Paris
& 0.8
& 15
& 0.03
& 0.03
& 0.70 \\

Shanghai
& 1.0
& 8
& 0.04
& 0.04
& 0.75 \\

Los Angeles
& 1.0
& 9
& 0.04
& 0.05
& 0.75 \\

London
& 1.0
& 10
& 0.04
& 0.05
& 0.76 \\

\bottomrule
\end{tabular*}

\endgroup
\end{table}

\noindent\textbf{\textsc{ECHO-Synth} Training Setup.}
\label{app:echo-setup}

We evaluate on the \textsc{ECHO-Synth} diameter (\textsc{Diam}), eccentricity (\textsc{Ecc}), and single-source shortest-path (\textsc{Sssp}) tasks~\citep{miglior2026echo}. We use the benchmark data splits and task-specific target scaling provided by \textsc{ECHO}. \textsc{Diam} is a graph-level regression task, while \textsc{Ecc} and \textsc{Sssp} are node-level regression tasks. Models are selected using validation loss and evaluated using test MSE and MAE.

All models use sequence length $L=40$, hidden dimension $d=64$, $4$ Mamba2 processing blocks, SSM state dimension $64$, hypernetwork hidden dimension $128$, and structural memory window $M=2$. Since $L>20$, we use the normalized structural recurrence in Eq.~\ref{eq:shadow-norm}. The task-specific learning rates are summarized in Table~\ref{tab:echo-config}.

\begin{table}[H]
\captionsetup{skip=6pt}
\caption{
Training configurations for the \textsc{ECHO-Synth} tasks. Architectural settings are shared across tasks, while the learning rate is selected separately for each task.
}
\label{tab:echo-config}
\centering

\begingroup
\small
\setlength{\tabcolsep}{7pt}
\renewcommand{\arraystretch}{1.08}

\begin{tabular}{@{}lcccccc@{}}
\toprule
\textsc{Task}
& $L$
& Blocks
& $d$
& $d_{\mathrm{state}}$
& $M$
& Learning Rate \\
\midrule

\textsc{Diam}
& 40
& 4
& 64
& 64
& 2
& $1\times10^{-4}$ \\

\textsc{Ecc}
& 40
& 4
& 64
& 64
& 2
& $1\times10^{-4}$ \\

\textsc{Sssp}
& 40
& 4
& 64
& 64
& 2
& $3\times10^{-4}$ \\

\bottomrule
\end{tabular}

\endgroup
\end{table}

We train with AdamW using batch size $32$, zero weight decay, zero dropout, no learning-rate scheduler, and gradient clipping with maximum norm $1.0$. Training runs for at most $200$ epochs with early stopping on validation loss and patience $100$. The checkpoint with the lowest validation loss is used for final validation and test evaluation. Training uses distributed data parallelism over $16$ GPUs, with batch size $32$ per device.

\newpage

\noindent\textbf{City-Networks Training Setup.}
\label{app:city-setup}

We evaluate on the Paris, Shanghai, Los Angeles, and London graphs from City-Networks~\citep{liang2026towards}. Each city is a single large graph for node classification, and we use the provided training, validation, and test splits. Node features are standardized using statistics from the training nodes.

All models use hidden dimension $d=128$, sequence length $L=17$, $4$ Mamba2 processing blocks, SSM state dimension $32$, structural memory window $M=4$, hypernetwork hidden dimension $256$, and $16$ learned seed queries. We use an attention-based readout across the hop sequence.

Exact feature attention in Eq.~\ref{eq:feature-attention} scales quadratically with the number of nodes, requiring $\mathcal{O}(Ln^2d)$ computation and $\mathcal{O}(n^2)$ attention memory. This is prohibitively expensive for \textsc{City-Networks}, where individual graphs contain hundreds of thousands of nodes. We therefore use WildCat~\citep{schroder2026wildcat}, which approximates softmax attention using a small weighted coreset selected by randomly pivoted Cholesky. For a coreset of size $r\ll n$, this replaces dense attention with approximately $\mathcal{O}\!\left(L(nrd+nr^2)\right)$ computation and $\mathcal{O}(nr)$ attention memory. Combined with the sparse structural recurrence, the sequence extractor requires approximately $\mathcal{O}\!\left(LM|\mathcal{E}|d + L(nrd+nr^2)\right)$ computation. We also use bfloat16 precision, gradient checkpointing, and chunked Mamba2 processing to reduce memory use.

We train with AdamW using a learning rate of $3\times10^{-4}$, weight decay $10^{-3}$, dropout $0.4$, and gradient clipping with maximum norm $1.0$. The training objective combines cross-entropy with label smoothing $0.05$ and an equally weighted ordinal EMD loss. Models are trained for at most $20{,}000$ epochs and evaluated every $100$ epochs. We reduce the learning rate by a factor of $0.5$ after $5$ validation evaluations without improvement, with a minimum learning rate of $10^{-5}$, and stop after $30$ evaluations without a validation-accuracy improvement of at least $0.03$. We select the checkpoint with the highest validation accuracy and report test accuracy.

\noindent\textbf{\textsc{LRIM} Training Setup.}
\label{app:lrim-setup}

We evaluate on the hard $16\times16$ variant of \textsc{LRIM}~\citep{mathys2026lrim} with interaction decay $\sigma=0.6$, the strongest long-range setting in the benchmark. Each example is a periodic grid with $256$ nodes, and the task is node-level energy regression on the original target scale. We use $8{,}000$ graphs for training, $1{,}000$ for validation, and $1{,}000$ for testing.

All models use sequence length $L=32$, hidden dimension $d=64$, $8$ processing blocks, SSM state dimension $64$, and hypernetwork hidden dimension $128$. We use the normalized structural recurrence in Eq.~\ref{eq:shadow-norm} and evaluate memory windows $M\in\{2,4,8\}$, with results reported in Figure~\ref{fig:window-size}. The sparse structural recurrence requires $\mathcal{O}(LM|\mathcal{E}|d)$ computation, while the dense feature-attention correction requires an additional $\mathcal{O}(Ln^2d)$ computation and $\mathcal{O}(n^2)$ memory for the attention matrix. Thus, the extractor has computational complexity $\mathcal{O}(LM|\mathcal{E}|d+Ln^2d)$ in our implementation, with dense attention providing the main scalability bottleneck for large graphs.

We train with AdamW using a learning rate of $3\times10^{-4}$, batch size $32$, and gradient clipping with maximum norm $1.0$. Training runs for at most $300$ epochs with early stopping on validation loss and patience $100$. Each configuration is evaluated over three random seeds using the same training setup. We compute $\log_{10}(\mathrm{MSE})$ separately for each seed and report the mean and sample standard deviation across seeds.

\noindent\textbf{Architectural Ablations.}
\label{app:lrim-setup}

\emph{Neighbor-pooling ablation.}
We remove the neighborhood representation $\mathbf{H}_{\mathrm{nbr}}$ from the graph-summary input and construct $\mathbf{h}_{\mathrm{graph}}$ using only the initial node representations $\mathbf{H}$. This removes explicit one-hop structural information from the graph summary. Performance decreases slightly relative to the full model.

\emph{Hypernetwork ablation.}
We replace the coefficients generated by the hypernetwork with fixed coefficients corresponding to Chebyshev propagation. Chebyshev propagation provides a strong non-adaptive baseline because it preserves nonstationary spectral information across information depth, as shown in Theorem~\ref{thm-main-body:spectral-information-preservation}. Removing the adaptive hypernetwork leads to a substantially larger test error, indicating that conditioning the propagation rule on the input graph and information depth contributes meaningfully to performance.

\emph{Gate ablation.}
We remove the gated attention correction $\beta_k\mathbf{F}_k\mathbf{U}^{(k)}$ from the extracted state and retain only the structural recurrence. This isolates the contribution of feature-dependent interactions beyond the learned propagation coefficients. Removing this term produces a smaller but consistent degradation in performance.

\begin{figure}[H]
    \centering
    \includegraphics[width=\linewidth]{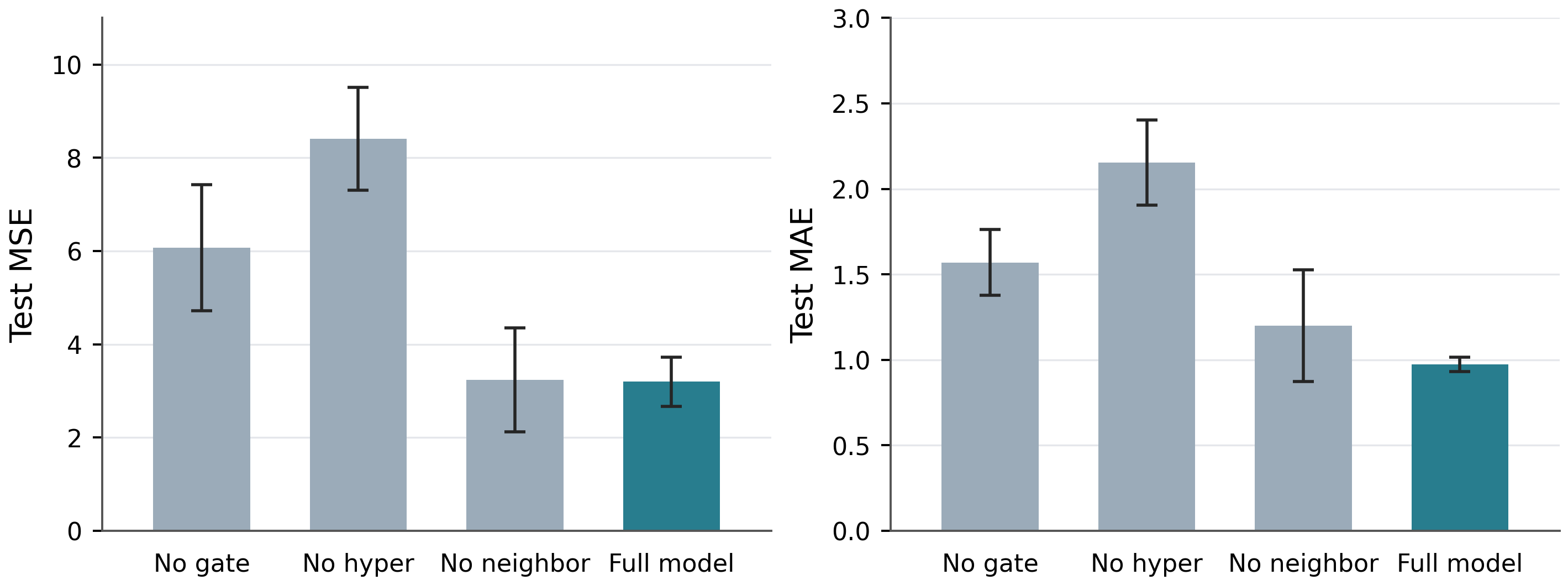}
    \caption{
    \textsc{ECHO-Synth} diameter task ablations. Bars show mean test error over three seeds, with error bars indicating $\pm 1$ sample standard deviation.
    }
    \label{fig:diam-ablations}
\end{figure}

\emph{Comparison on \textsc{LRIM-16}.}
To compare propagation mechanisms under the same optimization budget, we train \textsc{HOPPER}, GPS, LGSM with non-backtracking propagation, and the fixed-Chebyshev variant of \textsc{HOPPER} using a maximum of $300$ epochs and early-stopping patience of $100$. The results are summarized in Table~\ref{tab:lrim-controlled}.

\begin{table}[H]
\captionsetup{skip=6pt}
\caption{
Comparison on \textsc{LRIM-16} under the same training budget. We report test $\log_{10}(\mathrm{MSE})$, where lower is better. Results are reported as mean $\pm$ standard deviation across seeds.
}
\label{tab:lrim-controlled}
\centering

\begingroup
\small
\setlength{\tabcolsep}{8pt}
\renewcommand{\arraystretch}{1.08}

\begin{tabular}{@{}lc@{}}
\toprule
\textsc{Model}
& $\log_{10}(\mathrm{MSE}_{\mathrm{test}})\downarrow$ \\
\midrule

\textsc{HOPPER} ($M=2$)
& $-3.4332 \pm 0.0576$ \\

\textsc{HOPPER} ($M=4$)
& $-3.4388 \pm 0.0283$ \\

\textbf{\textsc{HOPPER} ($M=8$)}
& $\mathbf{-3.4851 \pm 0.0564}$ \\

\textsc{HOPPER} ($M=16$)
& $-3.4784 \pm 0.0704$ \\

\textsc{HOPPER}-Chebyshev
& $-3.3875 \pm 0.0370$ \\

LGSM-NBT
& $-3.1164 \pm 0.0491$ \\

GPS
& $-3.7831 \pm 0.1125$ \\

\bottomrule
\end{tabular}

\endgroup
\end{table}

GPS achieves the lowest test error, while the full \textsc{HOPPER} model outperforms both LGSM-NBT and the fixed-Chebyshev variant across the reported memory settings. The comparison further separates the benefit of learned propagation from gains that could be attributed only to using a stronger fixed polynomial extractor.

\section{Theoretical Properties of \textsc{HOPPER}}
\label{sec:theoretical-properties}

We first prove that the learned graph summary is invariant to node permutations, which allows us to establish permutation equivariance of the full sequence extractor.

\begin{lemma}[Permutation invariance of the graph summary]
\label{lem:graph-summary-invariance}
Let $\mathbf{P}\in\mathbb{R}^{n\times n}$ be any permutation matrix, and define $\mathbf{A}'=\mathbf{P}\mathbf{A}\mathbf{P}^{\top}$, $\mathbf{D}'=\mathbf{P}\mathbf{D}\mathbf{P}^{\top}$, and $\mathbf{H}'=\mathbf{P}\mathbf{H}$. Let $\mathbf{h}_{\mathrm{graph}}(\mathbf{H},\mathbf{A},\mathbf{D})$ denote the graph summary obtained from $\mathbf{H}_{\mathrm{nbr}}=\mathbf{A}_{\mathrm{rw}}\mathbf{H}$, $\mathbf{H}_{\mathrm{str}}=[\mathbf{H}\mid\mathbf{H}_{\mathrm{nbr}}]$, the seed attention in Eq.~\ref{eq:pool-attention}, and $\mathbf{h}_{\mathrm{graph}}=\mathrm{vec}(\mathbf{H}_{\mathrm{lat}})$. Then
\[
\mathbf{h}_{\mathrm{graph}}(\mathbf{H}',\mathbf{A}',\mathbf{D}')
=
\mathbf{h}_{\mathrm{graph}}(\mathbf{H},\mathbf{A},\mathbf{D}).
\]
\end{lemma}

\begin{proof}[Proof of Lemma \ref{lem:graph-summary-invariance}]
By definition,
$\mathbf{A}'_{\mathrm{rw}}
=
\mathbf{P}\mathbf{A}_{\mathrm{rw}}\mathbf{P}^{\top}$.
Since $\mathbf{H}'=\mathbf{P}\mathbf{H}$,
\[
\mathbf{H}'_{\mathrm{nbr}}
=
\mathbf{A}'_{\mathrm{rw}}\mathbf{H}'
=
\mathbf{P}\mathbf{A}_{\mathrm{rw}}\mathbf{P}^{\top}\mathbf{P}\mathbf{H}
=
\mathbf{P}\mathbf{H}_{\mathrm{nbr}}.
\]
Hence
$\mathbf{H}'_{\mathrm{str}}
=
[\mathbf{H}'\mid\mathbf{H}'_{\mathrm{nbr}}]
=
\mathbf{P}\mathbf{H}_{\mathrm{str}}$,
which implies
$\mathbf{K}'_{\mathrm{s}}=\mathbf{P}\mathbf{K}_{\mathrm{s}}$
and
$\mathbf{V}'_{\mathrm{s}}=\mathbf{P}\mathbf{V}_{\mathrm{s}}$. The seed queries $\mathbf{Q}_{\mathrm{s}}$ do not depend on the node ordering, so
$\mathbf{Q}_{\mathrm{s}}(\mathbf{K}'_{\mathrm{s}})^{\top}
=
\mathbf{Q}_{\mathrm{s}}\mathbf{K}_{\mathrm{s}}^{\top}\mathbf{P}^{\top}$.

Since row-wise softmax commutes with permutations of its columns,
$\operatorname{softmax}_{\mathrm{row}}(\mathbf{Z}\mathbf{P}^{\top})
=
\operatorname{softmax}_{\mathrm{row}}(\mathbf{Z})\mathbf{P}^{\top}$.
Therefore,
\[
\begin{aligned}
\mathbf{H}'_{\mathrm{lat}}
&=
\operatorname{softmax}_{\mathrm{row}}\left(
\frac{\mathbf{Q}_{\mathrm{s}}\mathbf{K}_{\mathrm{s}}^{\top}\mathbf{P}^{\top}}{\sqrt{d}}
\right)
\mathbf{P}\mathbf{V}_{\mathrm{s}} \\
&=
\operatorname{softmax}_{\mathrm{row}}\left(
\frac{\mathbf{Q}_{\mathrm{s}}\mathbf{K}_{\mathrm{s}}^{\top}}{\sqrt{d}}
\right)
\mathbf{P}^{\top}\mathbf{P}\mathbf{V}_{\mathrm{s}} \\
&=
\mathbf{H}_{\mathrm{lat}}.
\end{aligned}
\]
Since $\mathbf{h}_{\mathrm{graph}}=\operatorname{vec}(\mathbf{H}_{\mathrm{lat}})$, we conclude that
$\mathbf{h}'_{\mathrm{graph}}=\mathbf{h}_{\mathrm{graph}}$.
\end{proof}

We now use Lemma~\ref{lem:graph-summary-invariance} to establish permutation equivariance of the full sequence extractor.

\begin{theorem}[Permutation equivariance of the sequence extractor]
\label{thm:permutation-equivariance}
Let $\mathbf{P}\in\mathbb{R}^{n\times n}$ be any permutation matrix, and define
$\mathbf{A}'=\mathbf{P}\mathbf{A}\mathbf{P}^{\top}$,
$\mathbf{D}'=\mathbf{P}\mathbf{D}\mathbf{P}^{\top}$, and
$\mathbf{H}'=\mathbf{P}\mathbf{H}$. Suppose the feature-attention maps are permutation equivariant, so that $\mathbf{Q}'_k=\mathbf{P}\mathbf{Q}_k$ and $\mathbf{K}'_k=\mathbf{P}\mathbf{K}_k$, and that $\mathrm{LayerNorm}$ is applied independently to each node with shared parameters. Then, for every $k\geq 0$,
$\mathbf{U}'^{(k)}=\mathbf{P}\mathbf{U}^{(k)}$ and
$\mathbf{S}'^{(k)}=\mathbf{P}\mathbf{S}^{(k)}$. Consequently,
\[
\mathrm{SEQ}_{\theta}
(\mathbf{P}\mathbf{H},
 \mathbf{P}\mathbf{A}\mathbf{P}^{\top},
 \mathbf{P}\mathbf{D}\mathbf{P}^{\top})
=
\mathbf{P}\,
\mathrm{SEQ}_{\theta}(\mathbf{H},\mathbf{A},\mathbf{D}),
\]
where $\mathbf{P}$ acts on the node dimension of the hop sequence.
\end{theorem}

\begin{proof}[Proof of Theorem \ref{thm:permutation-equivariance}]
By Lemma~\ref{lem:graph-summary-invariance}, $\mathbf{h}'_{\mathrm{graph}}=\mathbf{h}_{\mathrm{graph}}$. Since the hop embedding $\mathbf{e}_k$ is independent of the node ordering, Eq.~\ref{eq:hyper-coeffs} therefore yields identical coefficients $\boldsymbol{\alpha}^{(k)}_A$, $\boldsymbol{\alpha}^{(k)}_D$, $\boldsymbol{\alpha}^{(k)}_I$, and $\beta_k$ on the original and relabeled graphs.

We first show $\mathbf{U}'^{(k)}=\mathbf{P}\mathbf{U}^{(k)}$ by induction on $k$. The claim holds at $k=0$ since $\mathbf{U}'^{(0)}=\mathbf{H}'=\mathbf{P}\mathbf{H}=\mathbf{P}\mathbf{U}^{(0)}$. Moreover, $\mathbf{D}'=\mathbf{P}\mathbf{D}\mathbf{P}^{\top}$. Assuming $\mathbf{U}'^{(k-j)}=\mathbf{P}\mathbf{U}^{(k-j)}$ for the preceding states in the recurrence, Eq.~\ref{eq:shadow} gives
\[
\begin{aligned}
\mathbf{U}'^{(k)}
&=
\sum_{j=1}^{M}
\left(
\alpha_A^{(k)}[j]\mathbf{P}\mathbf{A}\mathbf{P}^{\top}
+
\alpha_D^{(k)}[j]\mathbf{P}\mathbf{D}\mathbf{P}^{\top}
+
\alpha_I^{(k)}[j]\mathbf{I}
\right)
\mathbf{P}\mathbf{U}^{(k-j)} \\
&=
\mathbf{P}
\sum_{j=1}^{M}
\left(
\alpha_A^{(k)}[j]\mathbf{A}
+
\alpha_D^{(k)}[j]\mathbf{D}
+
\alpha_I^{(k)}[j]\mathbf{I}
\right)
\mathbf{U}^{(k-j)} \\
&=
\mathbf{P}\mathbf{U}^{(k)}.
\end{aligned}
\]

Next, $\mathbf{Q}'_k=\mathbf{P}\mathbf{Q}_k$ and $\mathbf{K}'_k=\mathbf{P}\mathbf{K}_k$ imply
$\mathbf{F}'_k=\mathbf{P}\mathbf{F}_k\mathbf{P}^{\top}$, since row-wise softmax satisfies
$\operatorname{softmax}_{\mathrm{row}}(\mathbf{P}\mathbf{Z}\mathbf{P}^{\top})
=\mathbf{P}\operatorname{softmax}_{\mathrm{row}}(\mathbf{Z})\mathbf{P}^{\top}$.
Hence $\mathbf{U}'^{(k)}+\beta_k\mathbf{F}'_k\mathbf{U}'^{(k)}
=\mathbf{P}(\mathbf{U}^{(k)}+\beta_k\mathbf{F}_k\mathbf{U}^{(k)})$.
Finally, row-wise $\mathrm{LayerNorm}$ commutes with node permutations, so Eq.~\ref{eq:feature-attention} gives $\mathbf{S}'^{(k)}=\mathbf{P}\mathbf{S}^{(k)}$. The result follows for every hop $k$.
\end{proof}

Finally, permutation equivariance of the sequence extractor implies permutation invariance of graph-level predictions under a permutation-invariant readout.

\begin{corollary}[Permutation invariance of graph-level predictions]
\label{cor:graph-prediction-invariance}
Suppose the downstream sequence processor applies the same map
$\Phi:\mathbb{R}^{L\times d}\rightarrow\mathbb{R}^{L\times d}$
independently to each node, and let $\rho$ be a permutation-invariant graph-level readout. Define
$f(\mathbf{X},\mathbf{A})=g(\rho(\mathbf{Z}))$
for any prediction head $g$. Then, for every permutation matrix $\mathbf{P}$,
\[
f(\mathbf{P}\mathbf{X},\mathbf{P}\mathbf{A}\mathbf{P}^{\top})
=
f(\mathbf{X},\mathbf{A}).
\]
\end{corollary}

\begin{proof}
Since $\mathbf{H}=\mathbf{X}\mathbf{W}_{\mathrm{in}}$, permuting the nodes gives $\mathbf{H}'=\mathbf{P}\mathbf{H}$. By Theorem~\ref{thm:permutation-equivariance}, the extracted sequence is permutation equivariant. Applying the same map $\Phi$ independently to each node therefore gives $\mathbf{Z}'=\mathbf{P}\mathbf{Z}$. Since $\rho$ is permutation invariant,
\[
\rho(\mathbf{Z}')
=
\rho(\mathbf{P}\mathbf{Z})
=
\rho(\mathbf{Z}),
\]
and applying $g$ yields
$f(\mathbf{P}\mathbf{X},\mathbf{P}\mathbf{A}\mathbf{P}^{\top})
=
f(\mathbf{X},\mathbf{A})$.
\end{proof}

As a consequence, the graph-level prediction depends only on the underlying attributed graph and not on the arbitrary ordering of its nodes. In particular, isomorphic labeled representations of the same graph receive identical graph-level predictions.

We next characterize the polynomial structure and locality of the learned structural recurrence.

\begin{theorem}[Adaptive polynomial structure and finite-hop structural support]
\label{thm:adaptive-polynomial-locality}
Fix a graph $\mathcal{G}=(\mathcal{V},\mathcal{E})$ and node representations $\mathbf{H}$, and let
$\{\boldsymbol{\alpha}_A^{(k)},\boldsymbol{\alpha}_D^{(k)},\boldsymbol{\alpha}_I^{(k)}\}_{k=1}^{L-1}$
denote the coefficients generated by the hypernetwork. Define
$\boldsymbol{\Gamma}^{(0)}=\mathbf{I}$,
$\boldsymbol{\Gamma}^{(i)}=\mathbf{0}$ for $i<0$, and, for $k\geq 1$,
\[
\boldsymbol{\Gamma}^{(k)}
=
\sum_{j=1}^{M}
\left(
\alpha_A^{(k)}[j]\mathbf{A}
+
\alpha_D^{(k)}[j]\mathbf{D}
+
\alpha_I^{(k)}[j]\mathbf{I}
\right)
\boldsymbol{\Gamma}^{(k-j)}.
\]
Then $\mathbf{U}^{(k)}=\boldsymbol{\Gamma}^{(k)}\mathbf{H}$ for every $k\geq 0$. Conditional on the generated coefficients, $\boldsymbol{\Gamma}^{(k)}$ is a noncommutative polynomial in $\mathbf{A}$ and $\mathbf{D}$ of degree at most $k$. Moreover, if $d_{\mathcal{G}}(u,v)>k$, then $[\boldsymbol{\Gamma}^{(k)}]_{uv}=0$. Consequently,
\[
\mathbf{U}^{(k)}_u
=
\sum_{v:\,d_{\mathcal{G}}(u,v)\leq k}
[\boldsymbol{\Gamma}^{(k)}]_{uv}\mathbf{H}_v.
\]
Thus, for fixed hypernetwork coefficients, the structural operator $\boldsymbol{\Gamma}^{(k)}$ is linear in $\mathbf{H}$ and has support within the $k$-hop neighborhood. Since the coefficients depend on the global graph summary $\mathbf{h}_{\mathrm{graph}}$, this support property does not imply that the end-to-end map $\mathbf{H}\mapsto\mathbf{U}^{(k)}$ is $k$-local.
\end{theorem}

\begin{proof}[Proof of Theorem \ref{thm:adaptive-polynomial-locality}]
We first show that $\mathbf{U}^{(k)}=\boldsymbol{\Gamma}^{(k)}\mathbf{H}$. The claim is immediate for $k=0$, since $\mathbf{U}^{(0)}=\mathbf{H}=\boldsymbol{\Gamma}^{(0)}\mathbf{H}$, and both $\mathbf{U}^{(i)}$ and $\boldsymbol{\Gamma}^{(i)}$ vanish for $i<0$. Suppose the claim holds for all indices smaller than $k$. Applying the induction hypothesis to Eq.~\ref{eq:shadow} gives
\[
\begin{aligned}
\mathbf{U}^{(k)}
&=
\sum_{j=1}^{M}
\left(
\alpha_A^{(k)}[j]\mathbf{A}
+
\alpha_D^{(k)}[j]\mathbf{D}
+
\alpha_I^{(k)}[j]\mathbf{I}
\right)
\boldsymbol{\Gamma}^{(k-j)}\mathbf{H} \\
&=
\boldsymbol{\Gamma}^{(k)}\mathbf{H}.
\end{aligned}
\]
The same recursion shows inductively that $\boldsymbol{\Gamma}^{(k)}$ is a noncommutative polynomial in $\mathbf{A}$ and $\mathbf{D}$. Indeed, $\boldsymbol{\Gamma}^{(0)}=\mathbf{I}$ has degree zero, and multiplying any term of $\boldsymbol{\Gamma}^{(k-j)}$ by $\mathbf{A}$ or $\mathbf{D}$ increases its degree by at most one. Since $j\geq 1$, every resulting term has degree at most $(k-j)+1\leq k$.

It remains to establish the locality claim. We proceed again by induction. For $k=0$, $\boldsymbol{\Gamma}^{(0)}=\mathbf{I}$, so $[\boldsymbol{\Gamma}^{(0)}]_{uv}=0$ whenever $u\neq v$. Suppose that $[\boldsymbol{\Gamma}^{(r)}]_{uv}=0$ whenever $d_{\mathcal{G}}(u,v)>r$ for every $r<k$. Since $\mathbf{D}$ and $\mathbf{I}$ are diagonal, left multiplication by either operator does not enlarge the support of $\boldsymbol{\Gamma}^{(k-j)}$. Left multiplication by $\mathbf{A}$ can enlarge its support by at most one graph hop. Hence each term in the recursion for $\boldsymbol{\Gamma}^{(k)}$ can be nonzero at $(u,v)$ only if
$d_{\mathcal{G}}(u,v)\leq (k-j)+1\leq k$.
Therefore $[\boldsymbol{\Gamma}^{(k)}]_{uv}=0$ whenever $d_{\mathcal{G}}(u,v)>k$. 

Using $\mathbf{U}^{(k)}=\boldsymbol{\Gamma}^{(k)}\mathbf{H}$ and the support property above gives
$\mathbf{U}^{(k)}_u = \sum_{v:\,d_{\mathcal{G}}(u,v)\leq k}
[\boldsymbol{\Gamma}^{(k)}]_{uv}\mathbf{H}_v$.
For fixed hypernetwork coefficients, $\boldsymbol{\Gamma}^{(k)}$ is independent of $\mathbf{H}$, so the resulting map is linear in $\mathbf{H}$. In HOPPER, however, these coefficients are generated from $\mathbf{h}_{\mathrm{graph}}$, which itself depends on $\mathbf{H}$. Hence the unconditional map $\mathbf{H}\mapsto\mathbf{U}^{(k)}$ need not be linear.
\end{proof}

We next show that the structural recurrence recovers standard LGSM propagators and also supports stable Chebyshev propagation.

\begin{theorem}[Containment of fixed LGSM extractors and stable Chebyshev propagation]
\label{thm:fixed-extractor-containment}
Suppose $M\geq 2$ and consider the recurrence
\[
\mathbf{U}^{(k)}
=
\sum_{j=1}^{M}
\left(
\alpha_G^{(k)}[j]\mathbf{G}
+
\alpha_A^{(k)}[j]\mathbf{A}
+
\alpha_D^{(k)}[j]\mathbf{D}
+
\alpha_I^{(k)}[j]\mathbf{I}
\right)\mathbf{U}^{(k-j)},
\;\;
\mathbf{U}^{(0)}=\mathbf{H},
\;\;
\mathbf{U}^{(i)}=\mathbf{0}\ \text{for }i<0.
\]
Then the following propagation schemes are recovered as special cases.

\begin{enumerate}
    \item[(i)] \textbf{Normalized adjacency propagation.}
    Setting $\alpha_G^{(k)}[1]=1$ for every $k\geq 1$ and all remaining coefficients to zero gives $\mathbf{U}^{(k)}=\mathbf{G}^{k}\mathbf{H}$. Taking $\mathbf{G}=\mathbf{A}_{\mathrm{sym}}$ recovers the adjacency based LGSM sequence $\mathcal{P}_{\mathrm{adj}}$.

    \item[(ii)] \textbf{Non-backtracking propagation.}
    Setting $\alpha_A^{(1)}[1]=1$, $\alpha_A^{(2)}[1]=1$, and $\alpha_D^{(2)}[2]=-1$, and for every $k\geq 3$ setting $\alpha_A^{(k)}[1]=1$, $\alpha_D^{(k)}[2]=-1$, and $\alpha_I^{(k)}[2]=1$, with all remaining coefficients equal to zero, gives $\mathbf{U}^{(k)}=\mathbf{B}^{(k)}\mathbf{H}$, where $\{\mathbf{B}^{(k)}\}_{k\geq 0}$ is the non-backtracking family defined in Section~\ref{sec:method}.

    \item[(iii)] \textbf{Chebyshev propagation.}
    Suppose additionally that $\mathbf{G}$ is symmetric and $\sigma(\mathbf{G})\subseteq[-1,1]$. Setting $\alpha_G^{(1)}[1]=1$, and for every $k\geq 2$ setting $\alpha_G^{(k)}[1]=2$ and $\alpha_I^{(k)}[2]=-1$, with all remaining coefficients equal to zero, gives $\mathbf{U}^{(k)}=T_k(\mathbf{G})\mathbf{H}$, 
    where $T_k$ is the degree-$k$ Chebyshev polynomial of the first kind. Moreover, $\|T_k(\mathbf{G})\|_2\leq 1$, and hence $\|\mathbf{U}^{(k)}\|_F\leq\|\mathbf{H}\|_F$ for every $k\geq 0$. In particular, one may take $\mathbf{G}=\mathbf{A}_{\mathrm{sym}}$.
\end{enumerate}
\end{theorem}

\begin{proof}[Proof of Theorem \ref{thm:fixed-extractor-containment}]
For part (i), the specified coefficients reduce the recurrence to
$\mathbf{U}^{(k)}=\mathbf{G}\mathbf{U}^{(k-1)}$. Since
$\mathbf{U}^{(0)}=\mathbf{H}$, induction gives
$\mathbf{U}^{(k)}=\mathbf{G}^{k}\mathbf{H}$ for every $k\geq 0$.

For part (ii), the stated coefficients give
$\mathbf{U}^{(1)}=\mathbf{A}\mathbf{H}$ and
$\mathbf{U}^{(2)}=\mathbf{A}\mathbf{U}^{(1)}-\mathbf{D}\mathbf{U}^{(0)}
=(\mathbf{A}^2-\mathbf{D})\mathbf{H}$.
For every $k\geq 3$, the recurrence becomes
\[
\mathbf{U}^{(k)}
=
\mathbf{A}\mathbf{U}^{(k-1)}
-
(\mathbf{D}-\mathbf{I})\mathbf{U}^{(k-2)}.
\]
Thus, if $\mathbf{U}^{(k-1)}=\mathbf{B}^{(k-1)}\mathbf{H}$ and
$\mathbf{U}^{(k-2)}=\mathbf{B}^{(k-2)}\mathbf{H}$, then
\[
\mathbf{U}^{(k)}
=
\left(
\mathbf{A}\mathbf{B}^{(k-1)}
-
(\mathbf{D}-\mathbf{I})\mathbf{B}^{(k-2)}
\right)\mathbf{H}
=
\mathbf{B}^{(k)}\mathbf{H}.
\]
The result follows by induction.

For part (iii), the specified coefficients give
$\mathbf{U}^{(0)}=\mathbf{H}$,
$\mathbf{U}^{(1)}=\mathbf{G}\mathbf{H}$, and, for every $k\geq 2$,
$\mathbf{U}^{(k)}=2\mathbf{G}\mathbf{U}^{(k-1)}-\mathbf{U}^{(k-2)}$.
Since the Chebyshev polynomials satisfy
$T_0(x)=1$, $T_1(x)=x$, and
$T_k(x)=2xT_{k-1}(x)-T_{k-2}(x)$,
induction gives
$\mathbf{U}^{(k)}=T_k(\mathbf{G})\mathbf{H}$.

For stability, since $\mathbf{G}$ is symmetric, write
$\mathbf{G}=\mathbf{Q}\boldsymbol{\Lambda}\mathbf{Q}^{\top}$,
where $\mathbf{Q}$ is orthogonal and every eigenvalue
$\lambda_i\in[-1,1]$. By the spectral theorem,
$T_k(\mathbf{G})=\mathbf{Q}T_k(\boldsymbol{\Lambda})\mathbf{Q}^{\top}$, so
\[
\|T_k(\mathbf{G})\|_2
=
\max_i |T_k(\lambda_i)|
\leq 1.
\]
The final inequality follows from Lemma~\ref{lem:chebyshev-bound}, since
$T_k(\cos\theta)=\cos(k\theta)$ and
$|\cos(k\theta)|\leq 1$.
Consequently,
$\|\mathbf{U}^{(k)}\|_F
=
\|T_k(\mathbf{G})\mathbf{H}\|_F
\leq
\|T_k(\mathbf{G})\|_2\|\mathbf{H}\|_F
\leq
\|\mathbf{H}\|_F$.
This proves the result.
\end{proof}

We now characterize the largest first-arrival sensitivity attainable within the stable polynomial class $\mathcal{P}_r$. If $d_{\mathcal{G}}(u,v)=r$, then all polynomial terms of degree less than $r$ vanish in the $(u,v)$ entry, so the sensitivity depends only on the leading coefficient. We therefore use the following extremal property of Chebyshev polynomials.

\begin{lemma}[Extremal leading coefficient of Chebyshev polynomials]
\label{lem:chebyshev-leading-coefficient}
Let $p(x)=a_r x^r+\cdots+a_0\in\mathcal{P}_r$ for $r\geq 1$. Then $|a_r|\leq 2^{r-1}$. Moreover, equality is attained by the Chebyshev polynomial $T_r$.
\end{lemma}

\begin{proof}[Proof of Lemma \ref{lem:chebyshev-leading-coefficient}]
By Lemma~\ref{lem:chebyshev-bound}, $T_r\in\mathcal{P}_r$, and the Chebyshev recurrence shows that its leading coefficient is $2^{r-1}$. It remains to show that no polynomial in $\mathcal{P}_r$ has a larger leading coefficient.

Suppose, for contradiction, that $|a_r|>2^{r-1}$. Replacing $p$ by $-p$ if necessary, we may assume $a_r>2^{r-1}$, and define
$q(x)=(2^{r-1}/a_r)p(x)$. Then $q$ and $T_r$ have the same leading coefficient, while $\sup_{x\in[-1,1]}|q(x)|<1$. Hence $h=T_r-q$ has degree at most $r-1$.

For $j=0,\ldots,r$, let $x_j=\cos(j\pi/r)$. By Lemma~\ref{lem:chebyshev-bound}, $T_r(x_j)=(-1)^j$. Since $|q(x_j)|<1$, the values $h(x_j)$ alternate in sign. The intermediate value theorem therefore gives at least one root of $h$ between each pair of consecutive points $x_j$ and $x_{j+1}$, so $h$ has at least $r$ distinct roots. This contradicts $\deg(h)\leq r-1$. Thus $|a_r|\leq 2^{r-1}$, and equality is attained by $T_r$.
\end{proof}

Combining Lemma~\ref{lem:chebyshev-leading-coefficient}'s extremal property with graph-distance locality yields the following main result:

\begin{theorem}[Optimal stable long-range sensitivity]
\label{thm:optimal-long-range-sensitivity}
Let $\mathbf{G}\in\mathbb{R}^{n\times n}$ be symmetric with $\sigma(\mathbf{G})\subseteq[-1,1]$, and suppose $\mathbf{G}_{uv}=0$ whenever $u\neq v$ and $\{u,v\}\notin\mathcal{E}$. Let $u,v\in\mathcal{V}$ satisfy $d_{\mathcal{G}}(u,v)=r\geq 1$. Then
\[
\sup_{p\in\mathcal{P}_r}
\left\|
\frac{\partial [p(\mathbf{G})\mathbf{H}]_u}
{\partial \mathbf{H}_v}
\right\|_2
=
2^{r-1}\left|[\mathbf{G}^r]_{uv}\right|,
\]
and the supremum is attained by the Chebyshev propagator $p=T_r$.

In particular,
$\left\|\partial [T_r(\mathbf{G})\mathbf{H}]_u/\partial\mathbf{H}_v\right\|_2
=2^{r-1}|[\mathbf{G}^r]_{uv}|$,
whereas standard power propagation satisfies
$\left\|\partial [\mathbf{G}^r\mathbf{H}]_u/\partial\mathbf{H}_v\right\|_2
=|[\mathbf{G}^r]_{uv}|$.
Hence, whenever $[\mathbf{G}^r]_{uv}\neq 0$, Chebyshev propagation achieves a factor $2^{r-1}$ improvement in first-arrival sensitivity over power propagation while remaining non-expansive, i.e., $\|T_r(\mathbf{G})\|_2\leq 1$.
\end{theorem}

\begin{proof}[Proof of Theorem \ref{thm:optimal-long-range-sensitivity}]
Let $p(x)=a_r x^r+\cdots+a_0\in\mathcal{P}_r$. Since $d_{\mathcal{G}}(u,v)=r$ and $\mathbf{G}$ is supported on the edges of $\mathcal{G}$, we have $[\mathbf{G}^j]_{uv}=0$ for every $0\leq j<r$. Indeed, a nonzero contribution to $[\mathbf{G}^j]_{uv}$ would imply a walk from $v$ to $u$ using at most $j$ graph edges, contradicting $d_{\mathcal{G}}(u,v)=r$. Therefore,
$[p(\mathbf{G})]_{uv}=a_r[\mathbf{G}^r]_{uv}$.

Writing $\mathbf{Z}=p(\mathbf{G})\mathbf{H}$ gives
$\mathbf{Z}_u=\sum_{w\in\mathcal{V}}[p(\mathbf{G})]_{uw}\mathbf{H}_w$, and hence
\[
\frac{\partial \mathbf{Z}_u}{\partial \mathbf{H}_v}
=
[p(\mathbf{G})]_{uv}\mathbf{I}_d
=
a_r[\mathbf{G}^r]_{uv}\mathbf{I}_d.
\]
It follows that
$\left\|\partial \mathbf{Z}_u/\partial\mathbf{H}_v\right\|_2
=|a_r||[\mathbf{G}^r]_{uv}|$.
By Lemma~\ref{lem:chebyshev-leading-coefficient}, $|a_r|\leq 2^{r-1}$, and therefore
\[
\left\|
\frac{\partial [p(\mathbf{G})\mathbf{H}]_u}
{\partial \mathbf{H}_v}
\right\|_2
\leq
2^{r-1}\left|[\mathbf{G}^r]_{uv}\right|.
\]
The Chebyshev polynomial $T_r$ belongs to $\mathcal{P}_r$ and has leading coefficient $2^{r-1}$, so it attains equality.

Finally, taking $p(x)=x^r$ gives
$\left\|\partial [\mathbf{G}^r\mathbf{H}]_u/\partial\mathbf{H}_v\right\|_2
=|[\mathbf{G}^r]_{uv}|$.
Whenever $[\mathbf{G}^r]_{uv}\neq 0$, the ratio between Chebyshev and power-propagation sensitivities is therefore $2^{r-1}$. The bound $\|T_r(\mathbf{G})\|_2\leq 1$ follows from Lemma~\ref{lem:chebyshev-bound}.
\end{proof}

The next result applies this characterization to obtain corresponding guarantees for regular graphs.

\begin{corollary}[Long-range sensitivity on regular graphs]
\label{cor:regular-graph-sensitivity}
Let $\mathcal{G}$ be a $q$-regular graph, so that $\mathbf{A}_{\mathrm{sym}}=q^{-1}\mathbf{A}$. Let $u,v\in\mathcal{V}$ satisfy $d_{\mathcal{G}}(u,v)=r$, and let $N_r(u,v)$ denote the number of shortest paths of length $r$ from $v$ to $u$. Then
\[
\left\|
\frac{\partial [\mathbf{A}_{\mathrm{sym}}^r\mathbf{H}]_u}
{\partial \mathbf{H}_v}
\right\|_2
=
\frac{N_r(u,v)}{q^r},
\;\;
\left\|
\frac{\partial [T_r(\mathbf{A}_{\mathrm{sym}})\mathbf{H}]_u}
{\partial \mathbf{H}_v}
\right\|_2
=
\frac{N_r(u,v)}{2}\left(\frac{2}{q}\right)^r.
\]
In particular, if $q=2$ and $N_r(u,v)=1$, then the sensitivity of power propagation is $2^{-r}$, whereas the sensitivity of Chebyshev propagation is $1/2$, independent of $r$.
\end{corollary}

\begin{proof}[Proof of Corollary \ref{cor:regular-graph-sensitivity}]
Since $\mathbf{A}_{\mathrm{sym}}=q^{-1}\mathbf{A}$, we have
$[\mathbf{A}_{\mathrm{sym}}^r]_{uv}=q^{-r}[\mathbf{A}^r]_{uv}$.
The entry $[\mathbf{A}^r]_{uv}$ counts length-$r$ walks from $v$ to $u$. Since $d_{\mathcal{G}}(u,v)=r$, every such walk is a shortest path, and therefore $[\mathbf{A}^r]_{uv}=N_r(u,v)$. The result now follows directly from Theorem~\ref{thm:optimal-long-range-sensitivity}.
\end{proof}

We now analyze how spectral information is preserved across information depth. Repeated adjacency propagation suppresses nonstationary spectral components, while the downstream sequence model receives all intermediate hop states. We first establish the following property of Chebyshev propagation.

\begin{lemma}[Finite-depth average energy of Chebyshev modes]
\label{lem:chebyshev-average-energy}
For every $\lambda\in(-1,1)$ and $L\geq 1$, let $\theta=\arccos(\lambda)\in(0,\pi)$. Then
\[
\frac{1}{L}\sum_{k=0}^{L-1}T_k(\lambda)^2
=
\frac{1}{2}
+
\frac{\sin(L\theta)\cos((L-1)\theta)}
{2L\sin\theta}.
\]
Consequently,
\[
\left|
\frac{1}{L}\sum_{k=0}^{L-1}T_k(\lambda)^2
-
\frac{1}{2}
\right|
\leq
\frac{1}{2L\sqrt{1-\lambda^2}}.
\]
\end{lemma}

\begin{proof}[Proof of Lemma \ref{lem:chebyshev-average-energy}]
Write $\lambda=\cos\theta$ for some $\theta\in(0,\pi)$. By Lemma~\ref{lem:chebyshev-bound}, $T_k(\lambda)=\cos(k\theta)$. Using $\cos^2(x)=(1+\cos(2x))/2$ gives
\[
\frac{1}{L}\sum_{k=0}^{L-1}T_k(\lambda)^2
=
\frac{1}{2}
+
\frac{1}{2L}\sum_{k=0}^{L-1}\cos(2k\theta).
\]
Since $\theta\in(0,\pi)$, we have $\sin\theta\neq 0$. Using the finite trigonometric-sum identity,
\[
\sum_{k=0}^{L-1}\cos(2k\theta)
=
\frac{\sin(L\theta)\cos((L-1)\theta)}{\sin\theta},
\]
and therefore
\[
\frac{1}{L}\sum_{k=0}^{L-1}T_k(\lambda)^2
=
\frac{1}{2}
+
\frac{\sin(L\theta)\cos((L-1)\theta)}
{2L\sin\theta}.
\]
This proves the first claim. Moreover, since
$|\sin(L\theta)\cos((L-1)\theta)|\leq 1$,
\[
\left|
\frac{1}{L}\sum_{k=0}^{L-1}T_k(\lambda)^2
-
\frac{1}{2}
\right|
\leq
\frac{1}{2L|\sin\theta|}.
\]
Finally, $\sin\theta=\sqrt{1-\lambda^2}$ because $\lambda=\cos\theta$ and $\theta\in(0,\pi)$. Hence,
\[
\left|
\frac{1}{L}\sum_{k=0}^{L-1}T_k(\lambda)^2
-
\frac{1}{2}
\right|
\leq
\frac{1}{2L\sqrt{1-\lambda^2}}.
\]
\end{proof}

Lemma~\ref{lem:chebyshev-average-energy} shows that each nonstationary eigenmode retains approximately half of its squared energy on average across a finite hop sequence, with an explicit error that decreases as $1/L$. In contrast, adjacency powers suppress the same modes geometrically with information depth. We now extend these bounds to the full graph representation.

\begin{theorem}[Finite-depth preservation of spectral information]
\label{thm:spectral-information-preservation}
Let $\mathbf{A}_{\mathrm{sym}}=\mathbf{Q}\boldsymbol{\Lambda}\mathbf{Q}^{\top}$, and let $\boldsymbol{\Pi}$ denote the orthogonal projector onto the eigenspace associated with eigenvalue $1$. Assume every remaining eigenvalue satisfies $|\lambda_i|<1$, and define
$\mathbf{H}_{\perp}=(\mathbf{I}_n-\boldsymbol{\Pi})\mathbf{H}$ and
$\rho=\max_{\lambda_i\neq 1}|\lambda_i|<1$. For adjacency power propagation
$\mathbf{V}^{(\ell)}=\mathbf{A}_{\mathrm{sym}}^\ell\mathbf{H}$, for every $L\geq 1$,
\[
\frac{1}{L}
\sum_{\ell=0}^{L-1}
\|(\mathbf{I}_n-\boldsymbol{\Pi})\mathbf{V}^{(\ell)}\|_F^2
\leq
\frac{1-\rho^{2L}}{L(1-\rho^2)}
\|\mathbf{H}_{\perp}\|_F^2.
\]

For the Chebyshev sequence
$\mathbf{U}^{(\ell)}=T_\ell(\mathbf{A}_{\mathrm{sym}})\mathbf{H}$, for every $L\geq 1$,
\[
\left|
\frac{1}{L}
\sum_{\ell=0}^{L-1}
\|(\mathbf{I}_n-\boldsymbol{\Pi})\mathbf{U}^{(\ell)}\|_F^2
-
\frac{1}{2}\|\mathbf{H}_{\perp}\|_F^2
\right|
\leq
\frac{1}{2L\sqrt{1-\rho^2}}
\|\mathbf{H}_{\perp}\|_F^2.
\]
\end{theorem}

\begin{proof}[Proof of Theorem \ref{thm:spectral-information-preservation}]
Let $\mathbf{q}_1,\ldots,\mathbf{q}_n$ denote the orthonormal eigenvectors of $\mathbf{A}_{\mathrm{sym}}$, and write
$\mathbf{H}=\sum_{i=1}^{n}\mathbf{q}_i\mathbf{c}_i^{\top}$.
Since $\boldsymbol{\Pi}$ projects onto the eigenspace corresponding to $\lambda_i=1$, we have
$\mathbf{H}_{\perp}=\sum_{\lambda_i\neq 1}\mathbf{q}_i\mathbf{c}_i^{\top}$ and
$\|\mathbf{H}_{\perp}\|_F^2=\sum_{\lambda_i\neq 1}\|\mathbf{c}_i\|_2^2$. For adjacency-power propagation, $(\mathbf{I}_n-\boldsymbol{\Pi})\mathbf{V}^{(k)}
=
\sum_{\lambda_i\neq 1}
\lambda_i^k\mathbf{q}_i\mathbf{c}_i^{\top}$. 
Orthogonality of the eigenvectors therefore gives
\[
\|(\mathbf{I}_n-\boldsymbol{\Pi})\mathbf{V}^{(k)}\|_F^2
=
\sum_{\lambda_i\neq 1}
|\lambda_i|^{2k}\|\mathbf{c}_i\|_2^2
\leq
\rho^{2k}\|\mathbf{H}_{\perp}\|_F^2.
\]
Averaging over $k=0,\ldots,L-1$ and summing the geometric series yields
\[
\begin{aligned}
\frac{1}{L}\sum_{k=0}^{L-1}
\|(\mathbf{I}_n-\boldsymbol{\Pi})\mathbf{V}^{(k)}\|_F^2
\leq
\frac{\|\mathbf{H}_{\perp}\|_F^2}{L}
\sum_{k=0}^{L-1}\rho^{2k} 
=
\frac{1-\rho^{2L}}{L(1-\rho^2)}
\|\mathbf{H}_{\perp}\|_F^2.
\end{aligned}
\]

For Chebyshev propagation, the spectral theorem gives $(\mathbf{I}_n-\boldsymbol{\Pi})\mathbf{U}^{(k)}
=
\sum_{\lambda_i\neq 1}
T_k(\lambda_i)\mathbf{q}_i\mathbf{c}_i^{\top}$,
and hence $\|(\mathbf{I}_n-\boldsymbol{\Pi})\mathbf{U}^{(k)}\|_F^2
=
\sum_{\lambda_i\neq 1}
T_k(\lambda_i)^2\|\mathbf{c}_i\|_2^2$.
Averaging over $k$ gives
\[
\frac{1}{L}\sum_{k=0}^{L-1}
\|(\mathbf{I}_n-\boldsymbol{\Pi})\mathbf{U}^{(k)}\|_F^2
=
\sum_{\lambda_i\neq 1}
\left(
\frac{1}{L}\sum_{k=0}^{L-1}T_k(\lambda_i)^2
\right)
\|\mathbf{c}_i\|_2^2.
\]
Subtracting $\frac{1}{2}\|\mathbf{H}_{\perp}\|_F^2$ from both sides and using the triangle inequality yields
\[
\begin{aligned}
\left|
\frac{1}{L}\sum_{k=0}^{L-1}
\|(\mathbf{I}_n-\boldsymbol{\Pi})\mathbf{U}^{(k)}\|_F^2
-
\frac{1}{2}\|\mathbf{H}_{\perp}\|_F^2
\right|
&\leq
\sum_{\lambda_i\neq 1}
\left|
\frac{1}{L}\sum_{k=0}^{L-1}T_k(\lambda_i)^2
-
\frac{1}{2}
\right|
\|\mathbf{c}_i\|_2^2.
\end{aligned}
\]
By Lemma~\ref{lem:chebyshev-average-energy},
\[
\left|
\frac{1}{L}\sum_{k=0}^{L-1}T_k(\lambda_i)^2
-
\frac{1}{2}
\right|
\leq
\frac{1}{2L\sqrt{1-\lambda_i^2}}.
\]
Since $|\lambda_i|\leq\rho$ for every $\lambda_i\neq 1$, we have
$\sqrt{1-\lambda_i^2}\geq\sqrt{1-\rho^2}$. Therefore,
\[
\begin{aligned}
\left|
\frac{1}{L}\sum_{k=0}^{L-1}
\|(\mathbf{I}_n-\boldsymbol{\Pi})\mathbf{U}^{(k)}\|_F^2
-
\frac{1}{2}\|\mathbf{H}_{\perp}\|_F^2
\right|
\leq
\frac{1}{2L\sqrt{1-\rho^2}}
\sum_{\lambda_i\neq 1}\|\mathbf{c}_i\|_2^2 
=
\frac{1}{2L\sqrt{1-\rho^2}}
\|\mathbf{H}_{\perp}\|_F^2.
\end{aligned}
\]
This proves the result.
\end{proof}

This result also gives a finite-depth characterization of over-smoothing through the Dirichlet energy $\mathcal{E}(\mathbf{H})=\operatorname{Tr}(\mathbf{H}^{\top}\boldsymbol{\Delta}\mathbf{H})$ where $\boldsymbol{\Delta}=\mathbf{I}_n-\mathbf{A}_{\mathrm{sym}}$ denotes the normalized graph Laplacian.

\begin{corollary}[Finite-depth preservation of Dirichlet energy]
\label{cor:dirichlet-energy-preservation}
Under the assumptions of Theorem~\ref{thm:spectral-information-preservation}, let
$\mathbf{V}^{(k)}=\mathbf{A}_{\mathrm{sym}}^k\mathbf{H}$ and
$\mathbf{U}^{(k)}=T_k(\mathbf{A}_{\mathrm{sym}})\mathbf{H}$. Then, for every $L\geq 1$,
\[
\frac{1}{L}
\sum_{k=0}^{L-1}
\mathcal{E}(\mathbf{V}^{(k)})
\leq
\frac{1-\rho^{2L}}{L(1-\rho^2)}
\mathcal{E}(\mathbf{H}),
\]
whereas
\[
\left|
\frac{1}{L}
\sum_{k=0}^{L-1}
\mathcal{E}(\mathbf{U}^{(k)})
-
\frac{1}{2}\mathcal{E}(\mathbf{H})
\right|
\leq
\frac{1}{2L\sqrt{1-\rho^2}}
\mathcal{E}(\mathbf{H}).
\]
Thus, adjacency-power propagation loses nonstationary Dirichlet energy at rate $\mathcal{O}(1/L)$ on average across information depth, while the Chebyshev sequence retains approximately one half of the initial Dirichlet energy.
\end{corollary}

\begin{proof}[Proof of Corollary \ref{cor:dirichlet-energy-preservation}]
Write $\mathbf{H}=\sum_{i=1}^{n}\mathbf{q}_i\mathbf{c}_i^{\top}$ in the eigenbasis of $\mathbf{A}_{\mathrm{sym}}$. Since $\boldsymbol{\Delta}=\mathbf{I}_n-\mathbf{A}_{\mathrm{sym}}$, we have
$\mathcal{E}(\mathbf{H})=\sum_{\lambda_i\neq 1}(1-\lambda_i)\|\mathbf{c}_i\|_2^2$. The proof now follows exactly as in Theorem~\ref{thm:spectral-information-preservation}, with each spectral component weighted by $(1-\lambda_i)$. For adjacency-power propagation,
$\mathcal{E}(\mathbf{V}^{(k)})\leq \rho^{2k}\mathcal{E}(\mathbf{H})$, so averaging over $k=0,\ldots,L-1$ gives
\[
\frac{1}{L}\sum_{k=0}^{L-1}\mathcal{E}(\mathbf{V}^{(k)})
\leq
\frac{1-\rho^{2L}}{L(1-\rho^2)}
\mathcal{E}(\mathbf{H}).
\]
For Chebyshev propagation, applying Lemma~\ref{lem:chebyshev-average-energy} to each eigenmode and using $|\lambda_i|\leq\rho$ gives
\[
\left|
\frac{1}{L}\sum_{k=0}^{L-1}\mathcal{E}(\mathbf{U}^{(k)})
-
\frac{1}{2}\mathcal{E}(\mathbf{H})
\right|
\leq
\frac{1}{2L\sqrt{1-\rho^2}}
\mathcal{E}(\mathbf{H}).
\]
\end{proof}

\section{Technical Lemmas}

\begin{lemma}[Chebyshev trigonometric identity]
\label{lem:chebyshev-bound}
Let $\{T_k\}_{k\geq 0}$ denote the Chebyshev polynomials of the first kind, defined by
$T_0(x)=1$, $T_1(x)=x$, and $T_k(x)=2xT_{k-1}(x)-T_{k-2}(x)$ for $k\geq 2$.
Then, for every $\theta\in\mathbb{R}$ and $k\geq 0$,
$T_k(\cos\theta)=\cos(k\theta)$. Consequently,
$\sup_{x\in[-1,1]}|T_k(x)|\leq 1$ for every $k\geq 0$.
\end{lemma}

\begin{proof}[Proof of Lemma \ref{lem:chebyshev-bound}]
The identity is standard for Chebyshev polynomials of the first kind.  We include the argument for completeness.
The result is immediate for $k=0$ and $k=1$. Suppose it holds for $k-1$ and $k-2$. Then, using the Chebyshev recurrence and the identity
$2\cos(\theta)\cos((k-1)\theta)
=\cos(k\theta)+\cos((k-2)\theta)$, we obtain
\[
\begin{aligned}
T_k(\cos\theta)
&=
2\cos\theta\,T_{k-1}(\cos\theta)
-
T_{k-2}(\cos\theta) \\
&=
2\cos\theta\cos((k-1)\theta)
-
\cos((k-2)\theta) \\
&=
\cos(k\theta).
\end{aligned}
\]
The claim follows by induction. Finally, every $x\in[-1,1]$ can be written as
$x=\cos\theta$ for some $\theta\in[0,\pi]$, and hence
$|T_k(x)|=|\cos(k\theta)|\leq 1$.
\end{proof}

\end{document}